\documentclass[oneside,11pt]{article}
\usepackage{caption}
\usepackage{blindtext}

\usepackage[abbrvbib, preprint]{jcustom}

\usepackage{lastpage}

\usepackage[utf8]{inputenc} 
\usepackage[T1]{fontenc}           
\usepackage{booktabs} 
\usepackage{hyperref}

\hypersetup{
    colorlinks=true,
    allcolors=blue
}

\usepackage{amsfonts,amsmath,amsthm,amssymb}       
\usepackage{nicefrac}       
\usepackage{microtype}      
\usepackage{xcolor}         
\usepackage{pifont}
\usepackage{tikz}
\usepackage{subcaption}

\usetikzlibrary{arrows.meta}
\usepackage{wrapfig}
\usepackage[presets={vec-cev}]{letterswitharrows}
\usepackage{cleveref}
\AddToHook{cmd/appendix/before}{
    \crefalias{section}{appendix}%
    \crefalias{subsection}{appendix}
}
\usepackage{graphicx}

\newtheorem{proposition}{Proposition}
\newtheorem{lemma}{Lemma}
\theoremstyle{remark}
\newtheorem{remark}{Remark}
\newtheorem{example}{Example}

\DeclareMathOperator{\Tr}{Tr}

\definecolor{src}{RGB}{76,114,176}
\definecolor{tgt}{RGB}{221,132,82}

\tikzset{
  traj/.style={black!75, line width=0.9pt, -{Latex[length=1.8mm]}},
  noisytraj/.style={black!70, line width=0.8pt, -{Latex[length=1.6mm]}, opacity=0.75},
  margsrc/.style={fill=src!12, draw=src!55, line width=0.5pt},
  margtgt/.style={fill=tgt!12, draw=tgt!55, line width=0.5pt},
  sp/.style={fill=src, draw=src},
  tp/.style={fill=tgt, draw=tgt},
}

\title{First-Order Trajectory Matching: Fast Ensemble Predictions of Chaotic, Turbulent, Stochastic Systems}

\author{%
  \name Shreya Jha \\
  \addr Courant Institute of Mathematical Sciences, New York University\\
  \AND
  Timo Schorlepp\\
   \addr Courant Institute of Mathematical Sciences, New York University\\
   \AND
  Nicholas Geissler\\
   \addr Courant Institute of Mathematical Sciences, New York University\\
   \AND
  Jules Berman\\
   \addr Courant Institute of Mathematical Sciences, New York University\\
   \AND
  Benjamin Peherstorfer\\
   \addr Courant Institute of Mathematical Sciences, New York University
}

\begin{document}

\maketitle
\begin{abstract}
We introduce First-Order Trajectory Matching (FTM), a surrogate-modeling method that learns the first-order local transport of probability mass from trajectories of stochastic systems. By matching the symmetric first-order motion of trajectories, FTM learns the probability current velocity, whose flow preserves time marginals to match ensemble averages, while also capturing current-like trajectory quantities such as fluxes, circulations, and barrier-crossing currents.  FTM learns the current velocity directly from trajectories, avoiding drift, diffusion, and score estimation. Our stability analysis separates discretization error from sampling variance and shows that the one-step simulation-free FTM loss is stable when temporal resolution and sample size are properly balanced.  Across stochastic dynamical systems and PDE examples, we empirically demonstrate that FTM provides trajectory-aware ensemble predictions at low, deterministic-rollout cost. 
\end{abstract}

\section{Introduction}\label{sec:Intro}
\paragraph{Overview: Fast ensemble predictions of physical systems}
Many dynamical systems in scientific applications are stochastic, or effectively so, because of chaos, coarse graining, partial observation, or unresolved degrees of freedom. Numerical high-fidelity simulators for such systems are often available, but repeated traditional numerical simulation is too costly for outer-loop tasks such as design, inverse problems, control, and uncertainty quantification. This motivates fast surrogate models trained from trajectory data. In the setting of (effectively) stochastic systems, however, the surrogate should not merely predict a single trajectory; it should rapidly propagate ensembles of trajectories and estimate their statistics, which describe the range of possible outcomes \cite{annurev:/content/journals/10.1146/annurev-fluid-010719-060214,Karniadakis2021,sanderse2025closureSciML}.

\paragraph{Operator learning and best forecasting predictions that collapse to the conditional mean}
Neural operators and related surrogate models \cite{Lu2021,li2021fourier} 
are typically trained on state-transition data with a mean-squared error loss, which means that the optimal operator-learning map is the conditional expectation of the next state given the current one. This gives a fast model, but the resulting map is a point predictor. As a consequence, intrinsic variability is averaged out so that multiple possible future states from the same observed current state are mapped to a single mean outcome; see \Cref{fig:overview}. Thus, such models fail to represent the transport of probability mass that determines ensemble statistics of stochastic systems 
\cite{jiang2023training,pmlr-v235-chen24n,molinaro2025generativeaifastaccurate}.

\paragraph{Population dynamics match the evolution of time marginals only} Population-dynamics methods infer a flow that matches the evolution of time marginals, which ensures that the right amount of probability mass is present in the right regions of the state space at each time. Examples include action matching 
\cite{neklyudov2023action,berman2024parametric},
the discrete inverse continuity equation method 
\cite{blickhan-berman-stuart-etal:2025}, and 
JKO/Wasserstein-gradient-flow and optimal transport-based methods
\cite{tong2020trajectorynet,bunne2022proximal,terpin2024learning}. Matching time marginals is sufficient to approximate ensemble averages but it is insufficient to determine ensemble dynamics, because it does not specify how mass flows over time. For example, a statistically stationary nonequilibrium system may have fixed time marginals while probability mass circulates persistently, which is not captured by matching time marginals alone; see \Cref{fig:overview}.

\paragraph{Conditional diffusion- and flow-based models and their inference costs}
An alternative is to model the full conditional transition law using conditional diffusion- and flow-based generative models \cite{song-sohl-dickstein-kingma-etal:2021,lipman2023flow,albergo-vanden-eijnden:2022,pmlr-v235-chen24n,KOHL2026108641}. Such models represent a distribution over future states rather than a single deterministic prediction, and therefore can capture  stochastic variability and martingale-dominated quantities. This expressivity comes at a computational cost, typically requiring multiple neural-network evaluations or numerical solver steps to generate each next state. In autoregressive rollouts, these costs are incurred at every physical time step. Recent approaches such as consistency,  mean-flow, and other distillation models aim to reduce costs; however, these methods have been developed primarily for static generative modeling and in practice often still trade off sample quality against the number of inference steps \cite{song2023consistency,zhou2025inductivemomentmatching,geng2025mean,boffi2025how,deng2026generativemodelingdrifting}. Furthermore, representing the underlying dynamic measure transport with a robust one-step map is intrinsically limited 
\cite{tsimpos2026oneshotgenerativeflowsexistence}.

\begin{figure}
\centering
\includegraphics[width=1.0\columnwidth]{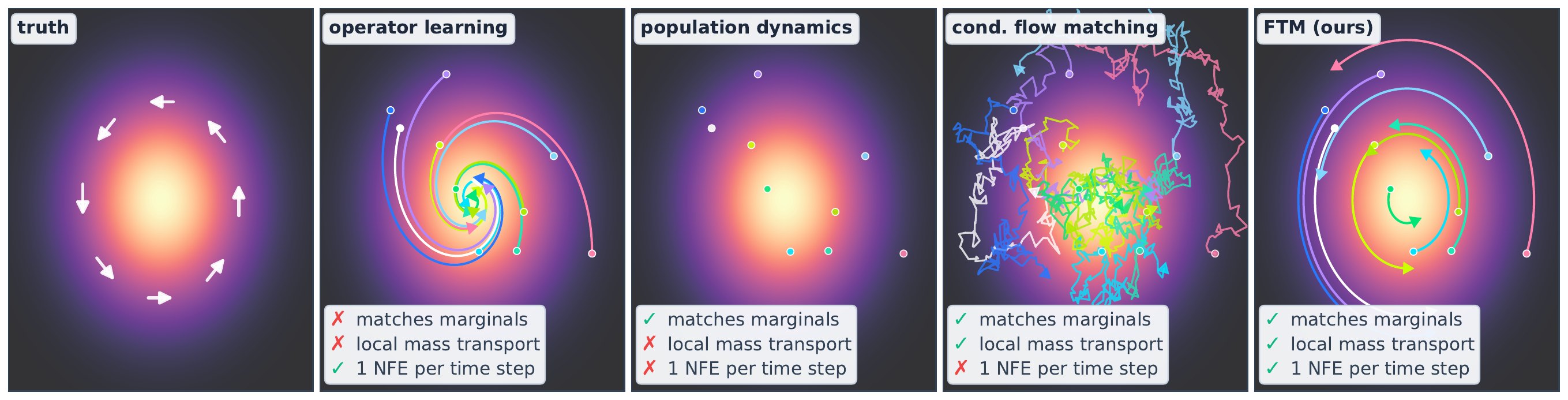}
\caption{Statistically stationary process where probability mass circulates. A learned operator model leads to trajectories that spiral inward, collapsing to the mean. Population dynamics (marginal matching) match time marginals but with zero velocity (for minimum kinetic energy) and so miss the rotation. Conditional generative modeling such as conditional flow matching reproduces the mass transport but typically requires multiple solver steps per physical time step. FTM preserves time marginals while capturing the first-order rotational current with one neural network evaluation (NFE\,$=1$) per time step. Details in  \Cref{appendix:IntroExample}.}
\label{fig:overview}
\end{figure}

\paragraph{Our approach: First-Order Trajectory Matching (FTM)}
We propose First-Order Trajectory Matching (FTM) which learns the local probability mass transport describing how probability mass moves through the state space. As illustrated in \Cref{fig:overview}, FTM learns a deterministic velocity field that transports ensembles in the right way both globally and locally. Globally, its flow matches the evolution of the time marginals, so probability mass is placed in the correct regions of state space. Locally, it matches the trajectory-induced transport of mass, preserving currents such as circulations or barrier-crossing fluxes. Computationally, FTM leads to a surrogate model based on an ordinary differential equation (ODE), so ensembles can be propagated by fast deterministic rollouts.
At the same time, this focus on a deterministic surrogate means that FTM is not intended to reproduce martingale-dominated path statistics such as hitting-time distributions or temporal correlations.

FTM matches local mass transport by using both directions of the observed trajectories:  Looking forward asks where trajectories go next and gives a one-sided forecasting velocity; looking backward asks where they came from. Combining these incoming and outgoing views gives the symmetric first-order motion of probability mass through a space-time point. In the infinitesimal limit, this is the probability current velocity, which is the learning target of FTM. 

The probability current velocity can be learned directly from stochastic trajectories, making FTM a fully trajectory-based surrogate-modeling approach. 
FTM uses a Stratonovich identity to turn current-velocity regression into a pathwise objective whose empirical form depends only on observed trajectory increments and evaluations of the candidate velocity field. We make this objective scalable by avoiding full-trajectory averaging, which is  too expensive for long rollouts and minibatch training. We therefore introduce a chunked loss that can balance between full-trajectory averaging and one-step local training, and prove a bound that separates the effects of time discretization, chunk length, and sample size. This analysis justifies the simple one-step FTM loss used in our experiments when time resolution and sample sizes are balanced. The one-step loss is local in time, cheap to evaluate, avoids separate drift/diffusion/score estimation, and remains stable when sufficiently many trajectories are available with respect to the time resolution.

\paragraph{Literature review}

\textbf{\emph{Probability current velocity in generative modeling}}
Probability-flow ODEs and probability current velocities also appear in diffusion- and flow-based generative modeling
\cite{song-sohl-dickstein-kingma-etal:2021,lipman2023flow,albergo-vanden-eijnden:2022,albergo-boffi-vanden-eijnden:2025,na-lee-yun-etal:2025},
where they describe the auxiliary sampling-time transport from a reference law to a target law and often provide faster alternatives to SDE samplers
\cite{chen-chewi-lee-etal:2023}.
In contrast, we are given physical-time trajectories of a stochastic system and seek the probability current velocity induced by those trajectories, not the velocity of an artificial path prescribed by a noising process or a stochastic interpolant.
Thus, unlike standard score-matching or conditional flow-matching objectives, our objective learns the physical-time probability current directly from observed paths.

\textbf{\emph{Probability current velocity in statistical physics and stochastic thermodynamics}}
Probability currents are classical objects in stochastic mechanics and nonequilibrium diffusion theory
\cite{nelson:1967,chetrite-gawedzki:2009,peliti-muratore-ginanneschi:2023}. 
Recent works in machine learning have used them in two main ways. First, in high-dimensional Fokker--Planck solvers, the probability current velocity is constructed from known drift-diffusion information together with an estimated score
and then used to evolve particles along the probability flow
\cite{maoutsa-reich-opper:2020,boffi-vanden-eijnden:2023,boffi-vanden-eijnden:2024}.
Second, in stochastic thermodynamics, there are methods for learning probability current velocities or related thermodynamic force fields, primarily for entropy-production diagnostics~\cite{seifert:2012,li2019quantifying,frishman-ronceray:2020,boffi-vanden-eijnden:2025,lyu-ray-crutchfield:2025,das2025localizing}. Most of this literature, which is also based on the Stratonovich identity~\eqref{eq:StratExp}, operates in the stationary regime, which excludes our setting of evolving $\rho(t)$. The work~\cite{frishman-ronceray:2020} gives a detailed analysis of trajectory length in this stationary regime, showing how longer averages reduce statistical error relative to the number of basis functions in linear approximations.  Recent work relaxes stationarity~\cite{lyu-ray-crutchfield:2025,das2025localizing}, but remains diagnostic, which means that the learned field is used to quantify or localize irreversibility and not integrated in time for predicting ensemble evolution. FTM instead develops a loss for learning a time-dependent velocity for predicting the evolution of a law $\rho(t)$. In particular, our analysis shows that, in this transient regime, finite-sample and discrete-time effects become central. We therefore introduce a chunk length that controls a statistical averaging scale, which in turn determines whether the empirical FTM loss provides a stable training signal for rollout.

\textbf{\emph{Learning SDEs}}
Classical SDE learning aims to estimate drift and diffusion coefficients, a long-standing problem that is difficult from finite, discretely observed trajectories and in high dimensions
\cite{nielsen-madsen-young:2000,comte-genon-catalot:2020,boninsegna-nueske-clementi:2018,ryder2018blackbox}.
Recent neural and latent SDEs parameterize these coefficients with neural networks 
\cite{tzen2019neural,li2020scalable,kidger2021efficient,course2023amortized,zhang2024trajectory,bartosh2025sdematching}. We
bypass the challenge of SDE identification by learning the probability current velocity directly from trajectories, without separately estimating drift, diffusion, or score.

\paragraph{Contributions}
(a) Formulate ensemble prediction as learning local probability mass transport rather than conditional means or marginal-matching flows, identifying the probability current velocity as the  deterministic surrogate for preserving  marginals and first-order trajectory transport.

(b) Derive and analyze a scalable, one-step, and simulation-free loss function to learn the current velocity directly from trajectories, bypassing drift, diffusion, and score estimation. 

(c) Show across stochastic and turbulent systems that FTM matches ensemble statistics and trajectory-based transport QoIs, while requiring only one neural-network evaluation per physical time step.

\section{Setup, preliminaries, and problem formulation}
\paragraph{Setup}
Let $(X(t))_{t \in [0,T]}$ be a stochastic process with state space $\mathbb{R}^d$, described by an SDE
\begin{equation}\label{eq:SDE}
\mathrm dX(t) = b(t, X(t))\mathrm dt + A(t)\mathrm dW(t)\,,\qquad X(0) \sim \rho(0)\,,
\end{equation}
with an initial probability density $\rho(0) \colon \mathbb{R}^d \to [0,\infty)$, a sufficiently smooth drift vector field $b \colon [0, T] \times \mathbb{R}^d \to \mathbb{R}^d$, a sufficiently smooth diffusion matrix $A(t) \in \mathbb{R}^{d \times d}$ that does not depend on $X(t)$ (``additive noise''), and a standard Wiener process $(W(t))_{t \in [0,T]}$ on $\mathbb{R}^d$. For all theoretical statements, we will further assume uniform boundedness of the drift and diffusion matrix: $\lVert b(t,x) \rVert_2 \leq b_{\max}$ for all $t \in [0,T],x \in \mathbb{R}^d$, and $\Sigma(t) := A(t) A(t)^\top \preceq \sigma_{\max} I_d$ for all $t \in [0,T]$. We denote the density of $X(t)$ as $\rho(t)\colon \mathbb{R}^d \to [0,\infty)$.
In the following, we are given a data set
\begin{equation}\label{eq:TrainingData}
\mathcal{X} = \{X^{(i)}(t_k) \,|\, i = 1, \dots, N, k = 0, \dots, K\} \subset \mathbb{R}^d\,,
\end{equation}
of $i = 1, \dots, N$ independent realizations of trajectories of~\eqref{eq:SDE} with $K$ time steps $0 = t_0 < t_1 < \dots, < t_K = T$ over the time interval $[0, T]$ and time-step size $h > 0$, assumed equidistant for simplicity (all of the following extends to data that are irregularly sampled in time).

\paragraph{Problem formulation}
Given the trajectory data~\eqref{eq:TrainingData}, we seek a deterministic velocity field $u: [0, T] \times \mathbb{R}^d \to \mathbb{R}^d$ whose induced ODE 
\begin{equation}\label{eq:ODEFlow}
\tfrac{\mathrm d}{\mathrm dt}\hat{x}(t) = u(t, \hat{x}(t))\,,\qquad \hat{x}(0) \sim \rho(0)\,,
\end{equation}
can be used as a fast surrogate for ensemble prediction under the stochastic dynamics \eqref{eq:SDE}. 
The basic requirement for such a surrogate is marginal consistency with the stochastic process \eqref{eq:SDE}, i.e., if $\hat{x}(0) \sim \rho(0)$, then the ODE should satisfy $\hat{x}(t) \sim \rho(t)$ for $t \in [0, T]$. 
Marginal consistency is necessary for ensemble prediction because it ensures that the learned flow places probability mass in the correct regions of state space at each time. 
We seek a deterministic velocity field whose  flow matches the time marginals of the stochastic process and additionally generates trajectories that are first-order consistent with sample paths of the SDE \eqref{eq:SDE}.

\section{First-order trajectory matching (FTM)}

With FTM, we do not only match where probability mass is globally but also how probability mass moves locally. Matching time marginals only ensures that the right amount of mass is globally present in the right regions of state space at each time. FTM additionally asks that, locally in space and time, the instantaneous transport of this mass agrees with that induced by  the stochastic trajectories.

\subsection{From global to local mass transport} 
\label{sec:motivate-current}

FTM seeks a velocity that describes how probability mass,  up to first order, passes through each space-time point $(t,x)$, rather than only where the mass is located at time $t$.

\paragraph{Describing local mass transport with forward, backward, and  symmetric increments}
A forward increment describes the first-order approximation of where trajectories currently at $x$ move mass next, 
\begin{equation}\label{eq:MeanFwdDrift}
\vec{b}_h(t, x) = \mathbb{E}_{\omega \sim \mathbb{P}}\left[h^{-1}\left({X_{\omega}(t + h) - X_{\omega}(t)}\right) \big| X_{\omega}(t) = x\right]\,,
\end{equation}
where $h > 0$ is a time-step size and $(X_{\omega}(t))_{t \in [0, T]}$ denotes a sample path of the SDE \eqref{eq:SDE}. In the limit, under suitable regularity, the expected forward increment $\vec{b}_h(t, x)$ converges point-wise to the drift of \eqref{eq:SDE}, $\vec{b}(t,x):=
\lim_{h \downarrow 0} \vec{b}_h(t, x) = b(t, x)$ \cite{kloeden-platen:1992}. However, local mass transport through $(t,x)$ is not only an outgoing quantity, it also depends on where the mass arriving at $(t, x)$ came from. The backward increment captures this incoming mass transport, \begin{equation}\label{eq:MeanBwdDrift}
\cev{b}_h(t, x) = \mathbb{E}_{\omega \sim \mathbb{P}}\left[h^{-1}\left({X_{\omega}(t) - X_{\omega}(t - h)}\right) \big| X_{\omega}(t) = x\right]\,,
\end{equation}
which converges to $\cev{b}(t,x):=\lim_{h \downarrow 0} \cev{b}_h(t, x) = b(t, x) - \Sigma(t)\nabla \log \rho(t, x)$ 
under suitable regularity assumptions \cite{nelson:1967,anderson:1982,haussmann-pardoux:1986}. With FTM, we aim to learn the velocity that describes the net flux of mass, which is obtained by combining the incoming and outgoing mass through the symmetric increment,
\begin{equation}\label{eq:SymIncrement}
v_h(t, x) = \mathbb{E}_{\omega \sim \mathbb{P}}\left[(2h)^{-1}\left(X_{\omega}(t + h) - X_{\omega}(t - h)\right) \big| X_{\omega}(t) = x\right].
\end{equation}
Combining the limits of $\vec{b}_h$ and $\cev{b}_h$ for $h \downarrow 0$ shows the limit of the expected symmetric increment $\lim_{h \downarrow 0} v_h(t, x) = \frac{1}{2}(\vec b(t,x)+\cev b(t,x))$~\cite{nelson:1967}.

\paragraph{Probability current velocity}
Using the forward and backward limits above, the limit velocity $\tfrac{1}{2}(\vec b(t,x)+\cev b(t,x))$ is
\begin{equation}\label{eq:CurrentVelocity}
    v(t,x)
    =
    b(t,x)
    -
    \tfrac{1}{2}\Sigma(t)\nabla \log \rho(t,x),
\end{equation}
which is the probability current velocity of the SDE~\eqref{eq:SDE}; see, e.g.,  \cite[Chapter~4]{pavliotis:2014}.

The importance of $v$ is that the same velocity selected locally by the trajectories also transports the marginals globally. Indeed, using~\eqref{eq:CurrentVelocity}, the Fokker--Planck
equation for~\eqref{eq:SDE} can be written as
\begin{equation}\label{eq:fokker-planck-current}
    \partial_t \rho(t,x)
    =
    - \nabla \cdot(b(t,x) \rho(t,x)) + \tfrac{1}{2} \Sigma(t) : \nabla^2 \rho(t,x)
    =
    -\nabla\cdot\left(\rho(t,x)v(t,x)\right).
\end{equation}
Equivalently, $\rho v$ is the probability current of the stochastic process~\eqref{eq:SDE}.
Thus, if $\hat{x}(0)\sim \rho(0)$ and $\hat{x}(t)$ evolves according to \eqref{eq:ODEFlow} with $u = v$, then the law of $\hat{x}(t)$ follows the same marginals as the SDE.

Therefore the probability current velocity satisfies exactly the two requirements we imposed. It matches the time marginals, so probability mass is globally in the correct places. It also arises as the symmetric first-order motion of the stochastic trajectories, so probability mass moves locally as it
does under the SDE. We discuss further properties of and intuition for $v$ in \Cref{appendix:IntroExample,appendix:prob-current}, and also refer to the large body of literature that uses the probability current velocity \cite{frishman-ronceray:2020,song-sohl-dickstein-kingma-etal:2021,boffi-vanden-eijnden:2023,lyu-ray-crutchfield:2025,boffi-vanden-eijnden:2025}. FTM learns the probability current velocity rather than an arbitrary marginal-matching velocity.

\subsection{A scalable loss function for learning the current velocity from trajectory data} 
At the core of FTM is the task of learning the probability current velocity $v$ directly from trajectory data. This is not a standard regression problem, because values of $v(t,x)$ are unavailable in the data set \eqref{eq:TrainingData}. Although $v$ can be represented as~\eqref{eq:CurrentVelocity}, estimating $v$ through separately learning drift, diffusion, and score is typically challenging and unreliable; see \Cref{sec:NumExp}. Instead, we are interested in a loss function whose minimizer is $v$, but whose empirical estimator uses only sample trajectories as given in the training data \eqref{eq:TrainingData}.

\paragraph{An inaccessible regression objective}
Let us  consider the parametrized velocity field $v_{\theta}: [0, T] \times \mathbb{R}^d \to \mathbb{R}^d$.
If the current velocity $v$ were known, we would have a standard regression objective
\begin{equation}\label{eq:CurrentVeloStart}
\min_{\theta \in \mathbb{R}^p} \mathbb{E}_{t \sim \mathcal{U}([0, T]), X(t) \sim \rho(t)}\left[\|v(t, X(t)) - v_{\theta}(t, X(t))\|^2_2\right]
\end{equation}
over the parameters $\theta \in \mathbb{R}^p$, where times $t$ are uniformly sampled  $\mathcal{U}([0, T])$ and $X(t)$ is a sample of the time marginal $\rho(t)$. 
We expand the norm in \eqref{eq:CurrentVeloStart} and drop the constant term to see that it is sufficient to minimize

\begin{equation}\label{eq:Loss:Intermediate01}
J(\theta) = \mathbb{E}_{t \sim \mathcal{U}([0, T]), X(t) \sim \rho(t)} \left[\|v_{\theta}(t, X(t))\|^2_2 - 2 \langle v(t, X(t)) , v_{\theta}(t, X(t))\rangle\right]\,.
\end{equation}
The quadratic term $\|v_{\theta}(t, X(t))\|^2_2$ can be directly estimated from data. The difficulty is the linear term $\langle v(t, X(t)) , v_{\theta}(t, X(t))\rangle$  because it contains the current velocity $v$, which is unknown. 

\paragraph{A pathwise loss via the Stratonovich identity}
The key observation is that we do not need to evaluate $v$ pointwise to estimate the linear term in the objective \eqref{eq:Loss:Intermediate01}. Instead, it is sufficient to be able to estimate its action on the parametrized velocity $v_{\theta}$. 
This action can be estimated directly from SDE trajectories through the Stratonovich integral identity
\begin{equation}\label{eq:StratExp}
\mathbb{E}_{\omega \sim \mathbb{P}} \left[\int_0^T \langle v_\theta(t, X_{\omega}(t)), v(t, X_{\omega}(t)) \rangle \mathrm{d} t\right] = \mathbb{E}_{\omega \sim \mathbb{P}} \left[ \int_0^T v_\theta(t, X_{\omega}(t)) \circ \mathrm{d} X_{\omega}(t) \right] 
\end{equation}
where $\circ$ denotes the Stratonovich dot product (see \Cref{appendix:StratDotProduct} for a proof of this standard result for completeness, and, e.g.,  \cite{frishman-ronceray:2020,boffi-vanden-eijnden:2025,lyu-ray-crutchfield:2025} where it is used for loss functions). 
The resulting objective is
\begin{equation}\label{eq:stratloss}
J(\theta) = J_T(\theta) = \mathbb{E}_{\omega \sim \mathbb{P}}\left[\frac{1}{T}\hspace*{-0.10cm}\int_0^T\|v_{\theta}(t, X_{\omega}(t))\|_2^2\mathrm dt - \frac{2}{T}\hspace*{-0.10cm}\int_0^T\ v_{\theta}(t, X_{\omega}(t)) \circ \mathrm dX_{\omega}(t)\right]\,,
\end{equation}
which does not depend on the unknown probability current velocity $v$ anymore. 
The objective has the same minimizer as the inaccessible regression loss \eqref{eq:CurrentVeloStart} but now depends only on trajectories.

\paragraph{A pathwise loss with chunks}
The loss \eqref{eq:stratloss} depends on the entire trajectory over the time interval $[0, T]$, which can become intractable for long trajectories. We therefore introduce a chunk length $\tau$ with $h \leq \tau \leq T$, and train on trajectory segments or chunks of length $\tau$. The chunked objective is 
\begin{align}\label{eq:strato-chunk-obj}
J_{\tau}(\theta) = \mathbb{E}_{t, \omega}\bigg[ \frac{1}{\tau}\hspace*{-0.10cm}\int_{t}^{t + \tau}\|v_{\theta}(s, X_{\omega}(s))\|_2^2\mathrm ds - \frac{2}{\tau}\hspace*{-0.10cm}\int_{t}^{t + \tau}v_{\theta}(s, X_{\omega}(s)) \circ \mathrm dX_{\omega}(s) \bigg]\,,
\end{align}
with $t \sim {\cal U}([0, T - \tau])$, which still has the same minimizers as $J(\theta)$. 
Given trajectories~\eqref{eq:TrainingData} sampled at discrete time points of time-step size $h > 0$, we set $\tau = K_{\tau}h$ for an integer $K_{\tau}$ and approximate the Stratonovich integral, e.g., by the composite trapezoidal rule to obtain the empirical loss 
\begin{multline}\label{eq:FTMLoss}
\widehat{J}_{h,\tau}(\theta) = \frac{1}{N\tau}\sum\nolimits_{i = 1}^N\sum\nolimits_{j=0}^{K_{\tau}-1} \left[\|v_{\theta}(s_{i,j}, X^{(i)}(s_{i,j}))\|_2^2h - \right.\\\left.\left\langle v_{\theta}\left(s_{i,j+1}, X^{(i)}(s_{i,j+1})\right) + v_{\theta}\left(s_{i,j}, X^{(i)}(s_{i,j}) \right), X^{(i)}(s_{i,j+1}) - X^{(i)}(s_{i,j})\right\rangle\right]\,,
\end{multline}
where $s_{i, j} = t_{m_i + j}$ for uniformly sampled chunk start indices $m_i \in \{0, \dots, K - K_{\tau}\}$.

\paragraph{Scalable one-step FTM loss when time discretization and sample size are balanced}
The objective $\widehat{J}_{h,\tau}$ depends on two time scales: the data time step $h$ and the chunk length $\tau$. The step size $h$ controls the bias from approximating the time integrals, while the chunk length $\tau$ controls how much stochastic path noise is averaged for one loss sample.  
The next bound makes this precise.
\begin{proposition} \label{prop:msebound}
Suppose $v_\theta \colon [0,T] \times \mathbb{R}^d \to \mathbb{R}^d$ is uniformly bounded, such that there exists $V > 0$ with $\lVert v_\theta(t,x) \rVert_2 \leq V$ for all $t \in [0,T]$ and $x \in \mathbb{R}^d$, and uniformly Lipschitz continuous, such that $C_{\text{div}}:= \sup_{t,x} \tfrac12 \lvert \Tr \left(\Sigma(t) \nabla v_\theta(t,x) \right) \rvert < \infty$. Then there exist constants $C, \tilde{C} > 0$ (related to the discretization error for the integrals) such that we have
\begin{equation}\label{eq:MSEBound}
\mathbb{E}\left[(\widehat{J}_{h,\tau}(\theta) - J_\tau(\theta))^2\right] \leq 2 \left(V^4 + 12 \left( V^2 b_{\max}^2 + \tfrac{\sigma_{\max} V^2}{\tau} + C_{\text{div}}^2 \right) + \tilde{C} h \right) \tfrac{1}{N}  + 4 C^2 V^2 h^2\,.
\end{equation}
\end{proposition}

Proof in \Cref{appendix:ProofOfMSEBound}. The bound separates the effect of the discretization scale $h$ from the effect of the averaging scale~$\tau$. 
The last term in the bound is the squared time-discretization bias. It decreases as the data are sampled more finely in time. 
The remaining terms are variance terms and scale like $1/N$ with the number of data samples $N$. Within this variance term, the contribution $\sigma_{\text{max}}V^2/(N\tau)$ is the stochastic contribution from the diffusive part of the trajectories. Increasing $\tau$ averages the noise over a longer trajectory chunk; increasing $N$ averages over more independent samples. Thus, stability (which means keeping the variance under control) is governed by both $N$ and $\tau$.
In particular, for fixed $\tau$, the variance remains bounded as $h \downarrow 0$. Thus, for a fixed $\tau$, the discretization bias can be reduced by taking smaller time steps $h$ without the variance growing unbounded. 
The variance-minimizing choice is $\tau = T$, but is often impractical because every training step requires
processing full trajectories, which is expensive for long time horizons and
limits minibatching.

We are therefore interested in  $\tau \approx h$, which couples the averaging scale~$\tau$ to the discretization scale~$h$, and leads to the scalable, one-step FTM loss
\begin{equation}\label{eq:FTM:OneStepEstimator}
\widehat{J}_{\text{FTM}}(\theta)\hspace*{-0.10cm} = \hspace*{-0.10cm}\frac{1}{Nh}\sum_{i = 1}^N\left[\|v_{\theta}(s_i, X^{(i)}(s_i))\|_2^2h - \langle v_{\theta}(s_i,X^{(i)}(s_i)), X^{(i)}(s_{i}\hspace*{-0.10cm}+\hspace*{-0.10cm}h) - X^{(i)}(s_i\hspace*{-0.10cm}-\hspace*{-0.10cm}h)\rangle\right],
\end{equation}
for uniformly sampled times $s_1, \dots, s_N \in \{t_1, \dots, t_{K-1}\}$. Here, we used the midpoint rule on chunks of size $\tau = 2h$ in~\eqref{eq:strato-chunk-obj}, instead of the trapezoidal rule; see \Cref{appendix:SymEstimatorUnstable} for further discussion and the relation of this loss to~\eqref{eq:SymIncrement}. 
The variance of the loss now behaves as $\sigma_{\max}V^2/(Nh)$. Hence, if $N$ is fixed, the estimator becomes unstable as $h \downarrow 0$. However, the bound in~\eqref{eq:MSEBound} shows that this instability is more nuanced. The variance of the one-step FTM loss \eqref{eq:FTM:OneStepEstimator} remains controlled whenever $Nh$ is sufficiently large. 

The one-step FTM loss is the most scalable one and we use it in our experiments. It is local in time, simulation-free, requires only adjacent states and a single neural-network evaluation per minibatch data point, and avoids estimating drift, diffusion, or scores. Its variance scales like $1/(Nh)$, so it is appropriate when the time-step size $h$ and the sample size $N$ are balanced.

\subsection{Time-marginal versus path-dependent inference errors}
At inference time, we generate samples by solving the ODE \eqref{eq:ODEFlow} with the learned $u = v_{\theta}$. We denote by $(\hat{\rho}_\theta(t))_t$ the marginal laws of the $(\hat{x}_\theta(t))_t$ trajectories generated via $v_\theta$. There are two inference errors of interest. 
First, the marginal-matching error, which measures how close  $\hat{\rho}_\theta(t)$  is to the true~$\rho(t)$, and therefore whether ensemble averages of ``time-local'' quantities of interests (QoIs) are accurately estimated. 
Second, the error for \textit{path-dependent QoIs}, which instead allows us to measure how accurately the learned ODE captures the path behavior of the SDE~\eqref{eq:SDE}.

\paragraph{Marginal matching inference error}
The following proposition about the marginal-matching error is based on a standard Gr{\"o}nwall argument and shows that uniform approximation of the current velocity implies Wasserstein control of the generated ensemble. A similar proposition holds for learning \textit{any} ODE~\eqref{eq:ODEFlow} for some $u \neq v$ that produces the correct time marginals.

\begin{proposition}[cf.\ e.g.~\cite{albergo-vanden-eijnden:2022,benton-deligiannidis-doucet:2023,fukumizu-suzuki-isobe-etal:2024}] \label{prop:w2}
Suppose that the true probability current velocity $v$ is uniformly Lipschitz-continuous with Lipschitz constant $L_v > 0$, and $\|v_{\theta}(t,x) - v(t,x)\|_2 \leq \epsilon$ for all $t \in [0,T],x \in \mathbb{R}^d$ for the learned velocity $v_\theta$. Then $W_2^2(\hat{\rho}_\theta(t), \rho(t)) \leq \exp((2L_v+1)t)t\epsilon^2$.
\end{proposition}
The proof is included in \Cref{appendix:ProofOfW2} for completeness. 
\Cref{prop:w2} immediately gives bounds for time-local QoIs that depend only on the marginal distribution. If $\phi: \mathbb{R}^d \to \mathbb{R}$ is a scalar QoI, which is Lipschitz with constant~$L_{\phi}$, then $|\mathbb{E}_{X(t) \sim \rho(t)}[\phi(X(t))] - \mathbb{E}_{\hat{x}_\theta(t) \sim \hat{\rho}_\theta(t)}[\phi(\hat{x}_\theta(t))]| \leq L_{\phi}\exp((L_v + \tfrac12)t)\sqrt{t}\epsilon$ (cf.\ \Cref{appendix:ProofOfW2}). 
Thus, if~$v_{\theta}$ approximates the probability current velocity~$v$ well, then the learned ODE produces accurate ensemble averages of local QoIs.

\paragraph{Path-dependent QoI inference error}
Marginal errors only measure where probability mass is located. By learning the current velocity, FTM also aims to match how probability mass moves, which means that sample trajectories generated with the learned ODE also approximate those of the SDE to first-order in the induced rate of change of probability mass; see, e.g., \cite{boffi-vanden-eijnden:2024} where this is used for capturing entropy production rates. The following proposition shows an error bound for path-dependent QoIs $\Phi(X) = \Phi((X(t))_{t \in [0,T]}):= \int_0^T \phi(t,X(t)) \circ \mathrm{d}X(t)$ for $\phi \colon [0,T] \times \mathbb{R}^d \to \mathbb{R}^d$ again under uniform $v_{\theta}$ error control; see \Cref{appendix:InferenceErrorRollout} for a proof using standard arguments. \textit{This} proposition is specific to FTM and the probability current velocity. In other words, targeting any other velocity field $u$ for~\eqref{eq:ODEFlow} with the same time-marginals does not permit control over the error in~$\Phi$, as is evident from the population dynamics example in \Cref{fig:overview}.

\begin{proposition}
\label{prop:path-error}
Under the same assumptions as \Cref{prop:w2}, suppose further that $v_{\theta}$ and $\phi$ are bounded, $\|v_{\theta}(t, x)\|_2 \leq B_{v_\theta}$ and $\lVert \phi(t,x) \rVert_2 \leq B_{\phi}$ for all $t \in [0,T]$, $x \in \mathbb{R}^d$, and Lipschitz on $\mathbb{R}^d$ with Lipschitz constants $L_{v_{\theta}}$ and $L_{\phi}$, respectively. Then, with $C = (L_{\phi} B_{v_\theta} + L_{v_{\theta}}B_{\phi}) / (L_v+\tfrac12)$:
\begin{align}
\left|
    \mathbb E[\Phi(X)]
    -
    \mathbb E[\Phi(\hat{x}_\theta)]
\right| \leq  \left(
\int_0^T \|\phi(t,\cdot)\|_{L^2(\rho(t))} \mathrm{d}t
+
C
\exp\left((L_v+\tfrac12)T\right)
\sqrt{T}\right) \epsilon \,.
\label{eq:qoi-rate-bound-w2}
\end{align}
\end{proposition}

The first term accounts for the local-in-time first-order range of change error. The second term is the marginal rollout error, which appears because the QoI is evaluated along samples from $\hat{\rho}(t)$ rather than $\rho(t)$, and is controlled by the Wasserstein bound of \Cref{prop:w2}.
The bound shows that models learned with FTM control both the quantities of interest under the learned ensemble and their
first-order range of change, capturing fluxes and circulations that are ignored by marginal matching.

\begin{figure}
    \centering
    \includegraphics[width=1.0\linewidth]{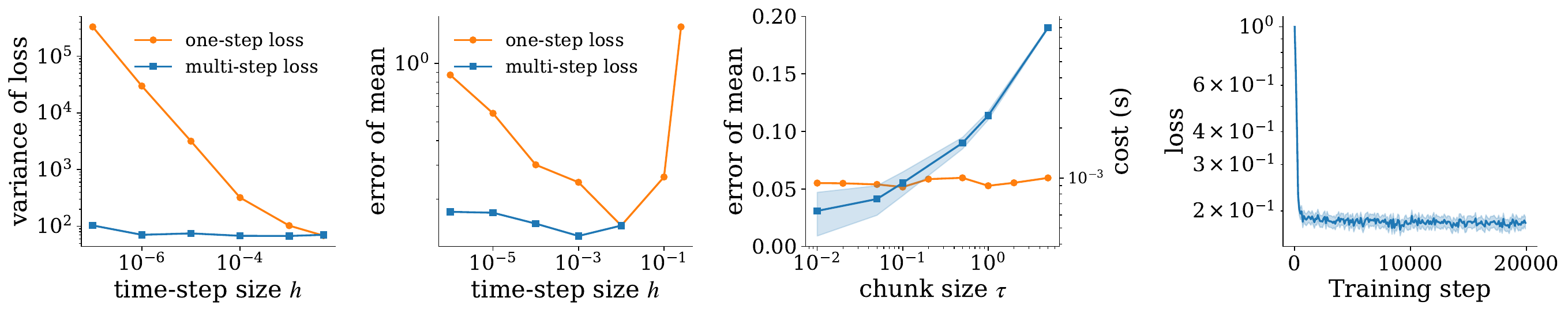}
    \caption{The one-step FTM loss is stable in the practically relevant time-resolution/sample-size balanced regime: longer chunks sizes $\tau$ are necessary only at very small time-step sizes $h$ to reduce variance (a)--(b), but otherwise add substantial cost per training step and provide little accuracy gain at realistic resolutions (here for Duffing oscillator example) (c). FTM trains robustly (d).} 
    \label{fig:duffing_combined}
\end{figure}

\section{Experiments}\label{sec:NumExp}

\paragraph{Examples} We consider the following examples; see \Cref{appendix:Experiments} for details. \textbf{1.\ Duffing oscillator.} Samples evolve in phase space based on a double-well potential. A Brownian forcing is  applied to the oscillator's acceleration equation. \textbf{2.\ Chaotic Rayleigh--Bénard convection.}  This is a  stochastically forced nine-mode model of Rayleigh--Bénard convection introduced in \cite{reiterer-lainscek-schuerrer-etal:1998}. The system is of dimension nine and chaotic due to the thermal convection;  all state variables are Brownian forced. \textbf{3.\ Stochastically forced Burgers.}  This system serves as a simplified model of turbulence, capturing nonlinear advective transport, viscous dissipation, and the formation of sharp gradients. The stochastic forcing is given by a 10-mode Fourier expansion driven by independent Brownian motions. \textbf{4.\ Stochastically forced turbulence.} This system is governed by the two-dim.\  incompressible Navier--Stokes equation;  Brownian forcing acts on the first ten Fourier modes.

\paragraph{The FTM loss is stable and robust to train}
\Cref{fig:duffing_combined} shows on the example of the Duffing oscillator that the one-step FTM loss is a principled scalable estimator, in agreement with \Cref{prop:msebound}. Panel~(a)--(b) show that for chunk size $\tau=h$, the variance/error grows only when $h$ becomes very small; fixing the chunk size $\tau=0.1 \gg h$ removes this small-$h$ blow-up.  
Crucially, Panel~(c) shows that in the time-resolution/sample-size balanced regime of our experiments (here for the Duffing oscillator), accuracy is nearly insensitive to $\tau$; as shown in  \Cref{prop:msebound}. 
Thus, longer chunks buy little accuracy in this balanced regime but incur substantially higher cost, whereas our  one-step FTM loss retains a stable training signal at minimal expense. Finally, Panel~(d) shows that this stability carries over to optimization. Across 100 random neural-network initializations, one-step FTM trains consistently and robustly. (The loss does not approach zero because  \eqref{eq:Loss:Intermediate01} drops the constant from  \eqref{eq:CurrentVeloStart}.)

\begin{figure}
    \begin{tabular}{cc}
\includegraphics[width=0.579\linewidth]{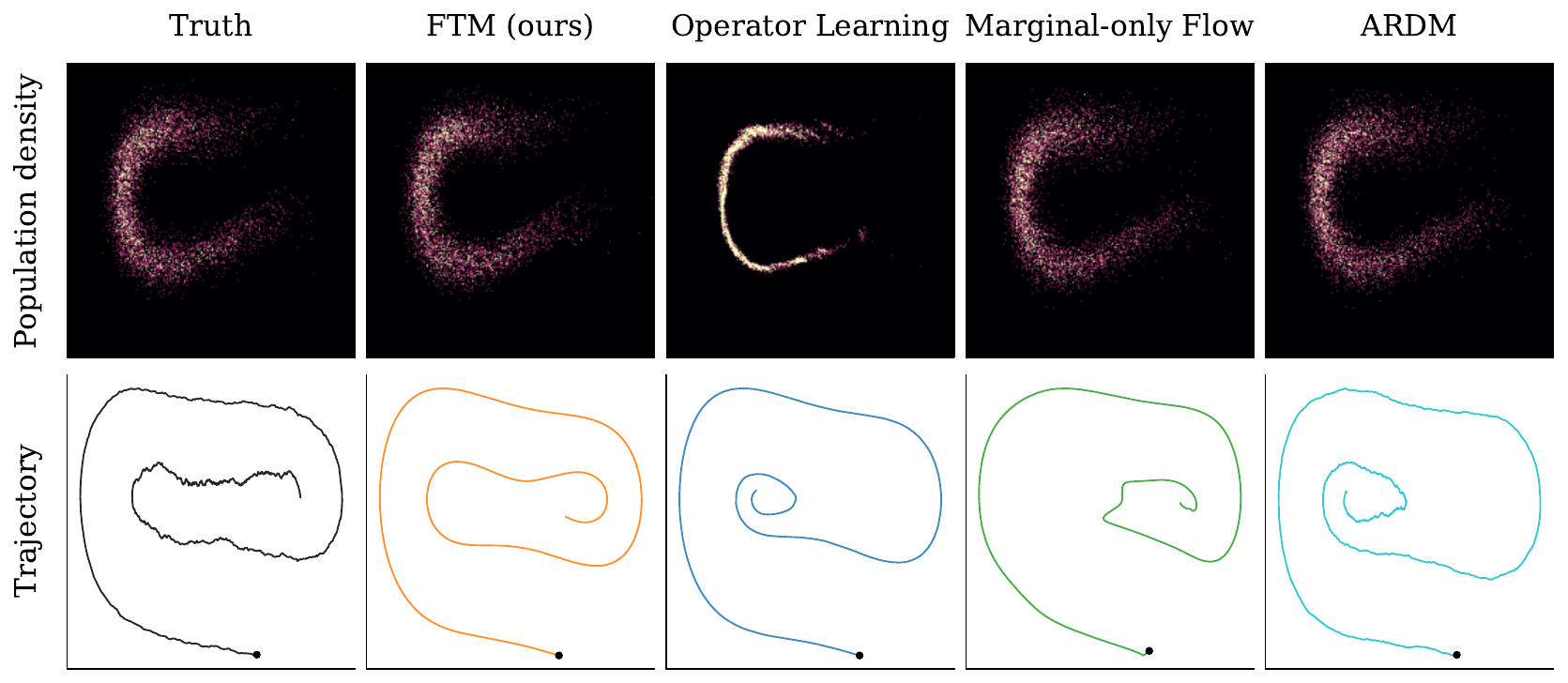}
& \includegraphics[width=0.37\linewidth]{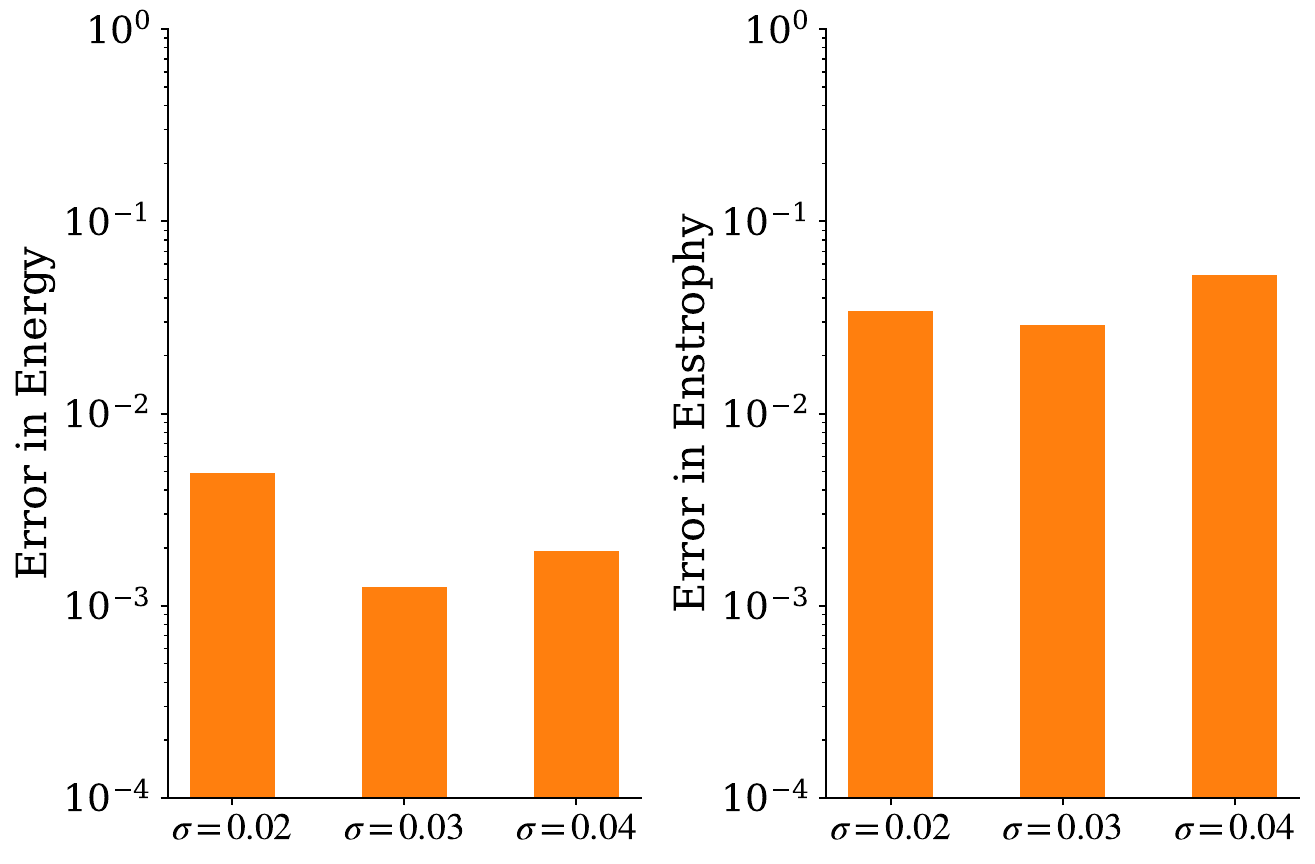}
\end{tabular}
    \caption{Duffing oscillator (left): FTM matches the evolving ensemble and generates physically meaningful trajectories. Stochastic Burgers (right): FTM is robust to increasing noise levels ($\sigma$).} 
    \label{fig:duffing_histogram}
\end{figure}

\begin{table}
\centering
\caption{FTM matches the ensemble distribution and additionally gives the smallest trajectory-QoI errors, showing that FTM captures transport information missed by marginal-only methods.} 
\label{tab:duffing_lorenz}
\setlength{\tabcolsep}{2.5pt}
\newcommand{\pmc}[2]{#1{(\pm#2)}}
\renewcommand{\arraystretch}{0.92}
\resizebox{0.9\columnwidth}{!}{%
\begin{tabular}{lllll}
\toprule
 & \multicolumn{2}{c}{Duffing oscillator} 
 & \multicolumn{2}{c}{Rayleigh--Bénard convection} \\
\cmidrule(lr){2-3} \cmidrule(lr){4-5}
 & dist.\ err.\ ($W_2$) & traj.\ QoI err.\ & dist.\ err.\ ($W_2$) & traj.\ QoI err.\ \\
\midrule
T.\ stepper \cite{otness2021an} 
& $\pmc{0.348}{2.0\mathrm{e}{-01}}$
& $\pmc{1.630}{8.4\mathrm{e}{-03}}$
& $\pmc{2.690}{1.6\mathrm{e}{+00}}$
& $\pmc{1.650}{3.8\mathrm{e}{-03}}$ \\

DICE \cite{blickhan-berman-stuart-etal:2025} 
& $\pmc{0.340}{5.0\mathrm{e}{-01}}$
& $\pmc{2.510}{7.6\mathrm{e}{-03}}$
& $\pmc{0.200}{1.2\mathrm{e}{-01}}$
& $\pmc{0.882}{3.8\mathrm{e}{-03}}$ \\

Marginal \cite{ho2020denoising} 
& $\pmc{0.081}{2.5\mathrm{e}{-02}}$
& $\pmc{0.564}{8.3\mathrm{e}{-03}}$
& $\pmc{0.111}{6.1\mathrm{e}{-02}}$
& $\pmc{1.270}{3.8\mathrm{e}{-02}}$ \\

SDE learning \cite{dridi2021learning} 
& $\pmc{0.083}{2.8\mathrm{e}{-02}}$
& $\pmc{0.263}{7.9\mathrm{e}{-03}}$
& $\pmc{\mathbf{0.054}}{2.3\mathrm{e}{-02}}$
& $\pmc{0.053}{4.8\mathrm{e}{-03}}$ \\

SDE matching \cite{bartosh2025sdematching} 
& $\pmc{0.549}{2.0\mathrm{e}{-01}}$
& $\pmc{0.178}{7.2\mathrm{e}{-03}}$
& $\pmc{0.227}{1.3\mathrm{e}{-01}}$
& $\pmc{0.612}{4.4\mathrm{e}{-03}}$ \\

Plug-in $v$\ (\ref{appendix:DirectEstimationOfV}) 
& $\pmc{0.507}{1.7\mathrm{e}{-01}}$
& $\pmc{1.150}{6.1\mathrm{e}{-03}}$
& $\pmc{0.374}{1.8\mathrm{e}{-01}}$
& $\pmc{0.566}{5.0\mathrm{e}{-03}}$ \\

FTM (ours) 
& $\pmc{\mathbf{0.075}}{4.9\mathrm{e}{-02}}$
& $\pmc{\mathbf{0.020}}{5.7\mathrm{e}{-03}}$
& $\pmc{0.060}{2.5\mathrm{e}{-02}}$
& $\pmc{\mathbf{0.027}}{4.6\mathrm{e}{-03}}$ \\
\bottomrule
\end{tabular}%
}
\end{table}

\paragraph{FTM captures trajectory-based transport beyond ensemble statistics}

\Cref{fig:duffing_histogram} (top row) and \Cref{tab:duffing_lorenz} (distribution error $W_2$) show that FTM accurately matches  time marginals and is competitive to marginal-only flows and explicit marginal-matching methods such as DICE \cite{blickhan-berman-stuart-etal:2025}.  
FTM additionally matches local mass transport: The bottom row of \Cref{fig:duffing_histogram} shows that FTM follows the local motion of the stochastic trajectories, whereas marginal-only transport produces paths that do not resemble the stochastic dynamics. To quantify this, we consider trajectory-based quantities of interest (QoI): for Duffing, we  measure the barrier-crossing current, and for Rayleigh--Bénard the rotational current (\Cref{appendix:Experiments}). FTM gives the smallest trajectory-based QoI error averaged over time in \Cref{tab:duffing_lorenz}.

\begin{figure}
\hspace*{-0.25cm}\begin{tabular}{cc}
\includegraphics[width=0.50\linewidth]{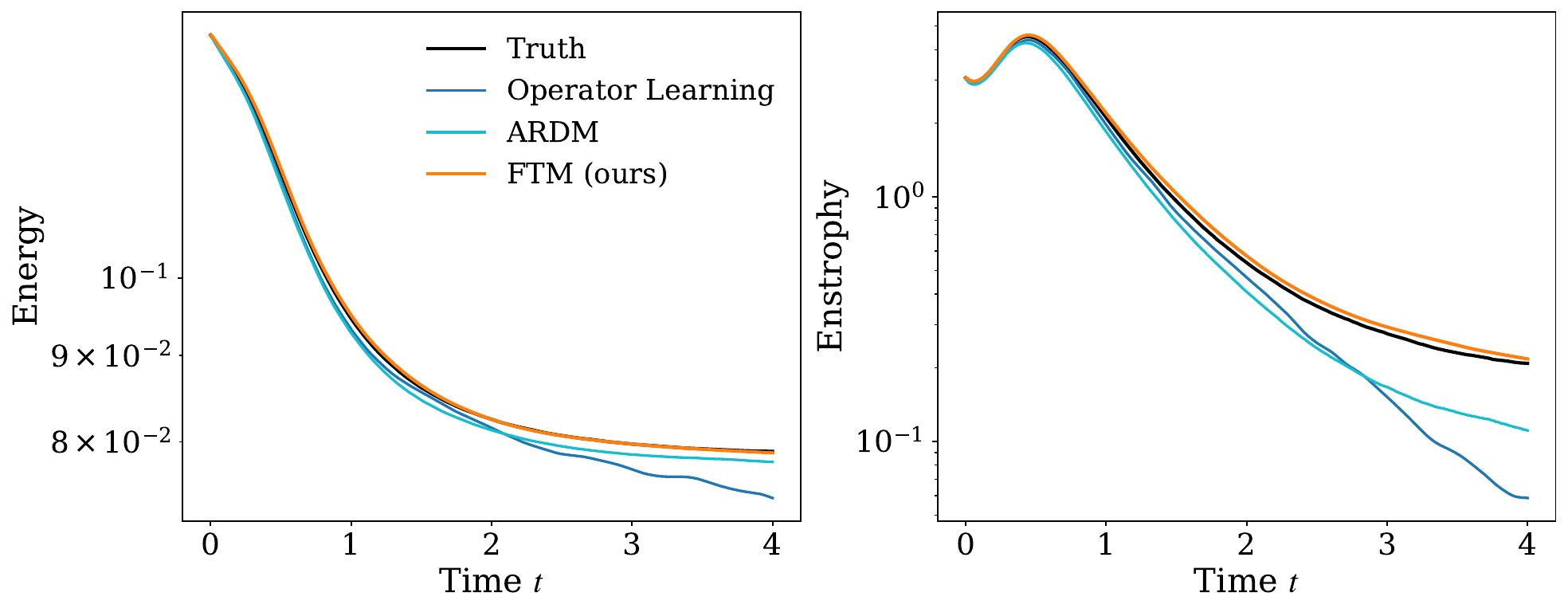} & \hspace*{-0.35cm}\includegraphics[width=0.50\linewidth]{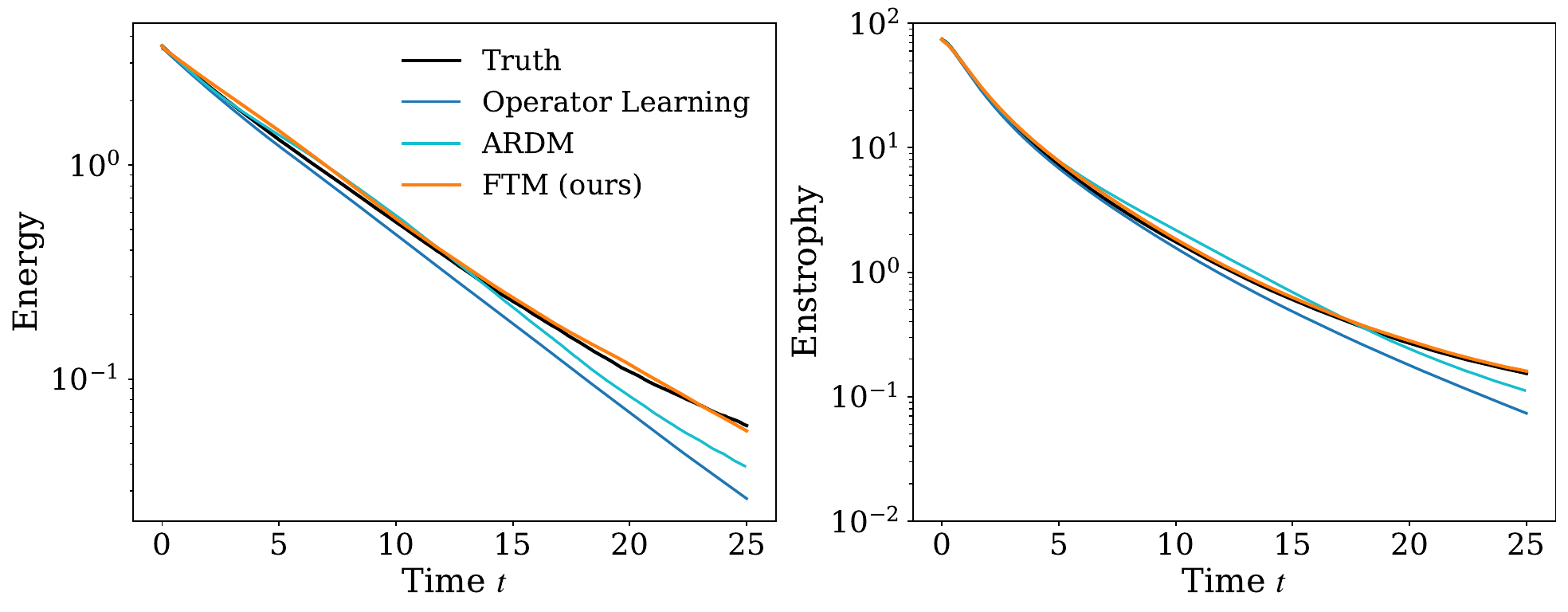}\\
\scriptsize (a) stochastic Burgers turbulence (Burgulence) & \scriptsize (b) stochastically forced turbulence (Navier-Stokes)
\end{tabular}
    \caption{FTM predicts ensembles with accurate energy and enstrophy statistics of solution trajectories of the stochastic Burgers (left panel) and stochastically forced turbulence (right panel) example. (See \Cref{fig:burgers_combined_appendix} in appendix for plots that include the standard deviation.)} 
    \label{fig:burgers_combined}
\end{figure}

\begin{table}
\centering
\caption{FTM gives the best reported errors while using only one neural-network function evaluation (NFE) per time step. More extensive results are provided in the appendix in \Cref{tab:burgers_jaxcfd}.}
\label{tab:maintextburgers_jaxcfd}
\setlength{\tabcolsep}{2.5pt}
\newcommand{\pmc}[2]{#1{(\pm#2)}}
\renewcommand{\arraystretch}{0.92}

\small\begin{tabular}{llrllr}
\toprule
 & \multicolumn{2}{c}{Stochastic Burgers}
 & \multicolumn{2}{c}{Forced turbulence (Navier--Stokes)} \\
\cmidrule(lr){2-3} \cmidrule(lr){4-5}
 & error in energy & NFE & error in energy & NFE \\
\midrule

Operator learning \cite{stachenfeld2021learned}
& $\pmc{2.21\mathrm{e}{-2}}{1.54\mathrm{e}{-2}}$
& 1
& $\pmc{2.13\mathrm{e}{-1}}{1.52\mathrm{e}{-1}}$
& 1 \\

ARDM 50 \cite{KOHL2026108641}
& $1.46\mathrm{e}{-2}(\pm\,3.39\mathrm{e}{-3})$
& 50
& $\pmc{3.14\mathrm{e}{-1}}{1.47\mathrm{e}{-1}}$ 
& 50 \\

ARDM 75 \cite{KOHL2026108641}
& $1.36\mathrm{e}{-2}(\pm\,3.49\mathrm{e}{-3})$
& 75
& $\pmc{2.75\mathrm{e}{-1}}{1.40\mathrm{e}{-1}}$  
& 75 \\

ARDM 100 \cite{KOHL2026108641}
& $\pmc{1.24\mathrm{e}{-2}}{3.12\mathrm{e}{-3}}$
& 100
& $\pmc{1.16\mathrm{e}{-1}}{1.09\mathrm{e}{-1}}$
& 100 \\

AR-CFM \cite{albergo-vanden-eijnden:2022,lipman2023flow}
& $\pmc{2.71\mathrm{e}{-3}}{1.96\mathrm{e}{-3}}$
& 20
& $\pmc{2.09\mathrm{e}{-1}}{1.81\mathrm{e}{-1}}$
& 100 \\

FTM (ours)
& $\pmc{\mathbf{1.95\mathrm{e}{-3}}}{1.93\mathrm{e}{-3}}$
& 1
& $\pmc{\mathbf{3.43\mathrm{e}{-2}}}{2.47\mathrm{e}{-2}}$
& 1 \\

\bottomrule
\end{tabular}%

\end{table}

\paragraph{FTM scales to stochastic PDEs with one NFE per physical time step}
We next demonstrate in \Cref{fig:burgers_combined,fig:jax_cfd_combined} and \Cref{tab:maintextburgers_jaxcfd} that FTM remains effective for the stochastic Burgers and stochastically forced turbulence (Navier-Stokes) example, where each state is a high-dimensional discretized field. Operator learning is included to verify that these benchmarks contain non-negligible stochastic variability. Indeed, operator learning damps the ensemble and underestimates both the energy and enstrophy. An auto-regressive diffusion model (ARDM) \cite{KOHL2026108641} avoids this collapse by learning the full transition law, but rollout then requires solving an inner generative sampling problem at every physical time step, with $100$ neural-network evaluations (NFEs) per time step in our experiments. FTM instead propagates ensembles with a single neural-network evaluation per physical time step while matching the mean and standard deviation of energy and enstrophy closely; see \Cref{fig:burgers_combined} (and \Cref{fig:burgers_combined_appendix} for standard deviation). 
We also report the energy spectrum decay in \Cref{fig:jax_cfd_combined}, which shows that FTM captures it accurately. We further compare to a\begin{table}
\small  \centering\caption{Stoch.\ Burgers: FTM vs. distilled CFM (MeanFlow \cite{geng2025mean})}\label{tab:MeanFlow}
  \hspace*{-0.25cm}\begin{tabular}{llr}
    & \hspace*{-0.3cm}err.\ energy & \hspace*{-0.2cm}NFE \\
    \hline
    AR-CFM \cite{albergo-vanden-eijnden:2022,lipman2023flow} & \hspace*{-0.3cm}2.71e-3& \hspace*{-0.2cm}20\\
    MeanFlow-10 \cite{geng2025mean} & \hspace*{-0.3cm}2.06e-2 & \hspace*{-0.2cm}10\\
    MeanFlow-5 \cite{geng2025mean}& \hspace*{-0.3cm}3.41e-2& \hspace*{-0.2cm}5\\
    MeanFlow-1\cite{geng2025mean} & \hspace*{-0.3cm}8.22e-1& 1\\
    FTM (ours) & \hspace*{-0.3cm}\textbf{1.95e-3}& \hspace*{-0.2cm}1
  \end{tabular}
\end{table} distilled auto-regressive conditional flow matching (ARCFM) model \cite{albergo-vanden-eijnden:2022,lipman2023flow} using mean-flow modeling \cite{geng2025mean}; see \Cref{tab:MeanFlow,tab:meanflowext}. While closer in terms of costs to FTM, this distillation leads to at least one order of magnitude higher error than FTM.

\begin{figure}
    \centering
    \includegraphics[width=0.95\linewidth]{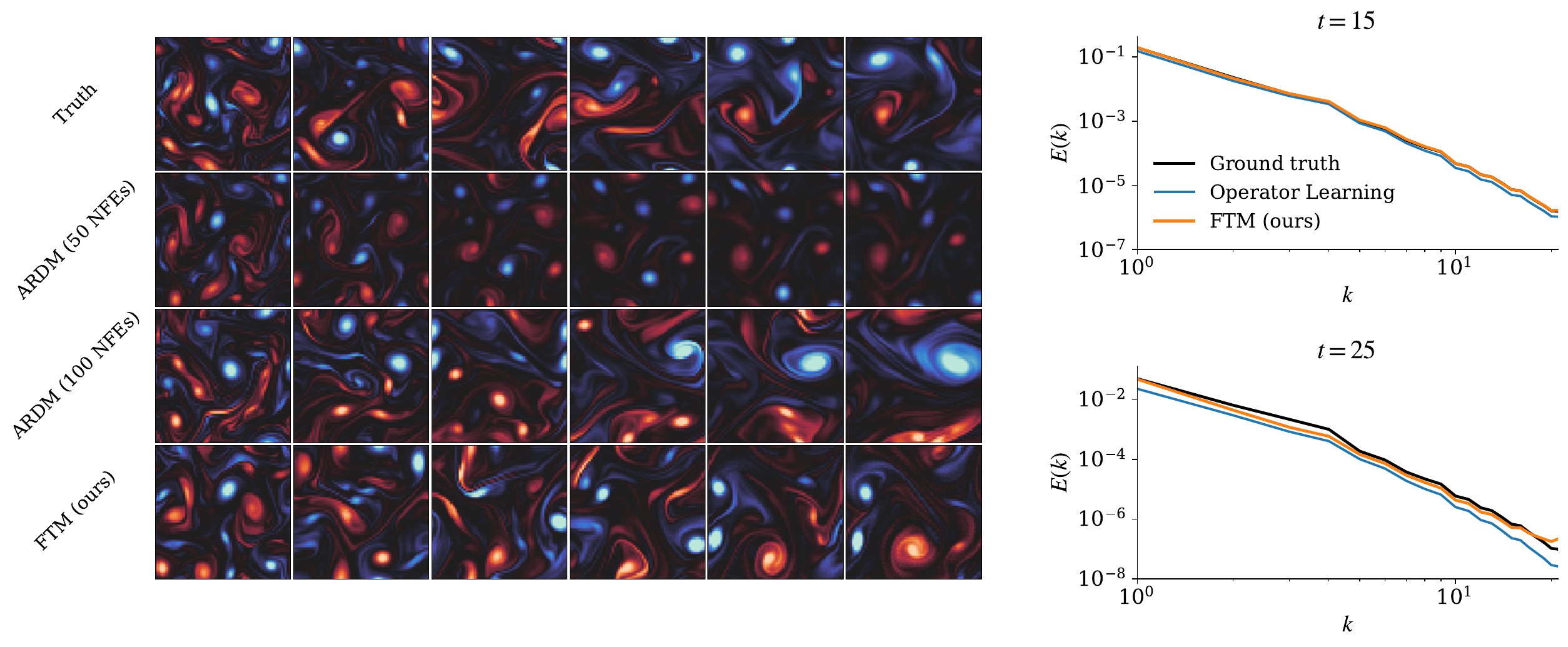}\vspace*{-0.3cm}
    \caption{On stochastically forced turbulence, FTM produces accurate rollouts with one NFE per physical time step, while ARDM uses up to $100$ inner sampling steps in our experiments (left). FTM also matches the energy spectrum decay well (right).}
    \label{fig:jax_cfd_combined}
\end{figure}

\section{Conclusions and limitations}\label{sec:Conc}
 
By learning the probability current velocity directly from trajectories, FTM captures how probability mass moves locally, not just where it is globally. This lets the FTM deterministic surrogate model predict  trajectory-based QoIs, including first-order transport structures such as fluxes and circulations, in addition to preserving ensemble statistics. As a result, FTM has the  rollout cost of deterministic operator learning while avoiding collapse to the mean, and is faster than typical autoregressive generative models that require inner sampling at every physical time step. 
The numerical experiments demonstrate that this makes FTM a practical route to fast  ensemble predictions.

\emph{Limitations} \textit{(i)} FTM is designed to preserve time marginals and first-order trajectory transport, not the full stochastic path law. Thus, FTM can accurately capture current-like path-dependent QoIs, as shown in the experiments, but it should not be expected to reproduce martingale-dominated statistics such as hitting-time distributions and temporal autocorrelations. This is the key tradeoff: when the quantities of interest are ensemble statistics or current-like transport observables, FTM can provide trajectory-aware predictions at deterministic-rollout costs, instead of paying the often much higher cost of autoregressive generative modeling. 
\textit{(ii)} The theory in this work establishes the estimator behavior for finite-dimensional additive-noise SDEs and clarifies the time-resolution/sample-size tradeoff of the one-step loss. These results, obtained using standard techniques, clarify the qualitative behavior of the estimator but they do not fully apply to the high-dimensional PDE/SPDE examples used in the empirical experiments. Extending the guarantees to broader the PDE/SPDE settings remains an important direction for future work.

\bibliography{main}

\appendix

\section{Illustrative examples, and gauge freedom for global mass transport}\label{appendix:IntroExample}

We first specify the SDE toy example shown in \Cref{fig:overview} in the main text via \Cref{example:ou-rotate}.

\begin{example} \label{example:ou-rotate}
    We consider a statistically stationary, rotating Ornstein--Uhlenbeck process in $d = 2$ dimensions. The SDE~\eqref{eq:SDE} satisfied by $(X(t))_{t \in [0,T]}$ is given by
    \begin{align}
    &\mathrm{d} X(t) =
    \mathrm{d} \begin{pmatrix}
        X_1(t)\\ X_2(t)
    \end{pmatrix} = b(t, X(t)) \mathrm{d}t + A(t) \mathrm{d}W(t)  \label{eq:ou-process-specs} \\
   &= - \left[\gamma X(t) + \Omega R X(t)\right] \mathrm{d}t + \sqrt{2D} \mathrm{d} W(t) = \begin{pmatrix}
       -\gamma & \Omega\\
       -\Omega & - \gamma
   \end{pmatrix} \begin{pmatrix}
        X_1(t)\\ X_2(t)
    \end{pmatrix} \mathrm{d} t + 
\sqrt{2D}   
   \begin{pmatrix}
        \mathrm{d}W_1(t) \\ \mathrm{d}W_2(t)
    \end{pmatrix}\,.
    \nonumber
    \end{align}
    Here, $\gamma > 0$ denotes the constant friction coefficient, $\Omega \in \mathbb{R}$ the constant angular velocity, and $D > 0$ the constant diffusion coefficient. We write $R = \begin{pmatrix}
        0 & -1\\
        1 & 0
    \end{pmatrix}$ in~\eqref{eq:ou-process-specs}. As is easy to verify from the Fokker--Planck equation~\eqref{eq:fokker-planck-current}, the stationary distribution of the process is $\rho_\infty = {\cal N}(0, \tfrac{D}{\gamma}I_2)$. We initialize $X(0) \sim \rho(0) = \rho_\infty$, which then implies $X(t) \sim \rho(t) \equiv \rho_\infty$ for all $t \in [0,T]$. Concretely, we take $\gamma = 0.35$, $\Omega = 1$, $D = 0.35$, $T = 1.5$ for \Cref{fig:overview}, though the numerical values do not matter for the qualitative behavior illustrated in \Cref{fig:overview}.
\end{example}

We now discuss what is shown in \Cref{fig:overview} in more detail, and separately comment on the five panels of \Cref{fig:overview} from left to right. Note that this illustrative figure, similar to \Cref{fig:two-lane} below, is generated based on analytically known results, without actually training on any data. All of the trajectories that are shown start from the \textit{same random initial conditions} $X(0) = \hat{x}(0) \sim \rho_\infty$ across panels, in order to compare the pathwise behavior of the different methods.
\begin{enumerate}
    \item \textbf{truth:} We show the constant marginal distribution $\rho(t) \equiv \rho_\infty$ of the process on $\mathbb{R}^2$ as a heatmap, with white arrows indicating the direction of the analytically known probability current velocity $v(x) = \Omega R x$ from~\eqref{eq:CurrentVelocity}, which circulates counterclockwise for $\Omega > 0$. 
This probability current corresponds to sample paths of~\eqref{eq:ou-process-specs} circulating in the same direction, while preserving their time marginal distribution. 
The same heatmap of $\rho_\infty$ is repeated in all other panels as a visual aid.
    \item \textbf{operator learning:} The target for operator learning is the drift vector field $b$ in~\eqref{eq:ou-process-specs} via~\eqref{eq:MeanFwdDrift}. 
Learning only the drift, and solving the ODE $\tfrac{\mathrm{d}}{\mathrm{d}t}\hat{x}(t) = b(\hat{x}(t))$ at inference, leads to trajectories that spiral inwards and collapse onto the mean $0$ of $\rho_\infty$ at large $t$. 
Hence, this does not preserve the correct time marginals.
    \item \textbf{population dynamics:}
    The main goal of population dynamics is to match the correct time marginals, which are stationary here. 
These methods typically assume that only \textit{unpaired} samples from $X(t)$ at different times $t$ are available, without trajectory information (which, as we argue in the present work, is actually available in many problems and useful to train with). Additional conditions then have to be imposed to single out a particular velocity field $u$ for the ODE~\eqref{eq:ODEFlow} among all those that yield the correct time marginals for $(\hat{x}(t))_t$. 
A natural condition is to demand minimum kinetic energy, resulting in $u(t,x) \equiv 0$ in the present example, which trivially satisfies $\hat{x}(t) = \hat{x}(0) \sim \rho_\infty$. However, this misses the circulation and local mass transport of the SDE~\eqref{eq:ou-process-specs} which is relevant for path-dependent quantities of interest.
The resulting ODE $\tfrac{\mathrm{d}}{\mathrm{d}t}\hat{x}(t) = 0$ is trivial to integrate in this particular example, but more generally the velocity is not zero and thus time stepping with population dynamics can require more than one neural-network evaluation per time step  as in our experiments.
    \item \textbf{conditional flow matching:} Conditional flow matching methods are capable of reproducing the correct transition law of $X(t+h) \mid X(t)$ of the SDE~\eqref{eq:ou-process-specs}. In the figure, we hence show sample paths of the SDE~\eqref{eq:ou-process-specs} generated using the Euler--Maruyama method. By construction, these will have the correct time marginals $\rho(t) \equiv \rho_\infty$, and circulate counterclockwise according to the probability current. However, sampling from (approximations of) $X(t+h) \mid X(t)$ with conditional flow matching in each time step increases the computational cost at inference.
    \item \textbf{FTM:} Here, we integrate the ODE $\tfrac{\mathrm{d}}{\mathrm{d}t}\hat{x}(t) = v(t,\hat{x}(t))$, which preserves the time marginals, can be done at low computational cost, and the resulting pure circulation at constant radius mimics the behavior of true sample paths of~\eqref{eq:ou-process-specs}.
\end{enumerate}

In order to illustrate and contextualize these points further, it is useful to provide a slightly more abstract view on \Cref{example:ou-rotate} using general irreversible perturbations of reversible processes that retain the Boltzmann equilibrium distribution:

\begin{example}[e.g.~\cite{rey-bellet-spiliopoulos:2015}]
    \label{example:ortho-comp}
    Consider an SDE~\eqref{eq:SDE} with a diffusion matrix $A(t) = \sqrt{2} I_d$ and drift $b(t,x) = - \nabla U(x) + \ell(x)$ for a (quasi-) potential $U \colon \mathbb{R}^d \to [0,\infty)$ such that $Z = \int_{\mathbb{R}^d} \exp(-U(x)) \mathrm{d}x < \infty$. Here, $\ell \colon \mathbb{R}^d \to \mathbb{R}^d$ denotes the  ``orthogonal component'' of the drift, and we suppose $\langle \nabla U(x), \ell(x) \rangle = 0$ and $\nabla \cdot \ell(x) = 0$ for all $x \in \mathbb{R}^d$. The vector field $\ell$ represents an irreversible perturbation of the SDE. Under the stated assumptions, the stationary distribution of the process $(X(t))$ is still the Boltzmann distribution $\rho_\infty(x) = Z^{-1} \exp (-U(x))$ as can be verified from the Fokker--Planck equation~\eqref{eq:fokker-planck-current}:
    \begin{align*}
        &\nabla \cdot \left(b(x) \rho_\infty(x) - \nabla \rho_\infty(x) \right) = Z^{-1} \nabla \cdot \left( \left[(-\nabla U(x) + \ell(x)) + \nabla U(x) \right] \exp (-U(x)) \right)\\
        &= Z^{-1} \nabla \cdot \left( \ell(x)  \exp (-U(x)) \right) = \rho_\infty(x) \left(\nabla \cdot \ell(x) - \langle \ell(x), \nabla U(x) \rangle \right) = 0\,.
    \end{align*}
    If we set $X(0) \sim \rho(0) = \rho_\infty$, then $\rho(t) = \rho_\infty$ for all $t > 0$. The probability current velocity~\eqref{eq:CurrentVelocity} in this case is
    \begin{align}
        v(t,x) = b(t,x) - \nabla \log \rho(t,x) = -\nabla U(x) + \ell(x) + \nabla U(x) = \ell(x)\,,
        \label{eq:prob-velo-orthogonal-comp}
    \end{align}
    i.e., it corresponds to only the rotational or ``swirling'' component of the drift. Despite the time marginals $\rho(t) \equiv \rho_\infty$ being stationary, typical sample paths of the SDE~\eqref{eq:SDE} swirl according to~$\ell$, and this is correctly captured by the probability current velocity~\eqref{eq:prob-velo-orthogonal-comp}. The two-dimensional rotating Ornstein--Uhlenbeck process~\eqref{eq:ou-process-specs} as specified in \Cref{example:ou-rotate} and shown in \Cref{fig:overview} obviously falls exactly under the general class of examples introduced here: in this special case, we have $A(t) = \sqrt{2 D}I_2$ and $b(t,x) = - \gamma x + \Omega R x$, such that $U(x) = \tfrac{\gamma}{2} \lVert x \rVert_2^2$ and $\ell(x) = v(x) = \Omega R x$. Evidently, $\langle \nabla U, \ell \rangle = 0$ and $\nabla \cdot \ell = 0$, such that $\rho_\infty \propto \exp(-U / D)$, and the probability current velocity corresponds to circulation with constant angular velocity $\Omega$.
\end{example}

 In view of the discussion above and in \Cref{sec:motivate-current} in the main text, we remark that in general there is a ``gauge freedom'' in choosing a velocity field~$u$ for~\eqref{eq:ODEFlow} with given time marginals. Indeed, as~\eqref{eq:fokker-planck-current} shows, for any $u$ that is a valid choice under the time marginal constraint, so is $u + w / \rho$ for $\nabla \cdot w = 0$. What we have argued in \Cref{sec:motivate-current} in the main text based on the local probability mass transport of the SDE is that picking $u = v$ as the probability current velocity~\eqref{eq:CurrentVelocity} is in some sense the canonical choice if we are given sample paths of the SDE~\eqref{eq:SDE}.
 
 Concretely, in the specific situation of \Cref{example:ou-rotate} and more generally \Cref{example:ortho-comp}, picking any other velocity field $u = v + w / \rho$ with $\nabla \cdot w = 0$ for the probability flow ODE~\eqref{eq:ODEFlow} that respects the time marginals would either introduce too much or too little swirl compared to the actual SDE~\eqref{eq:SDE}. In particular, since in these examples the system is statistically stationary, $u \equiv 0$ is a valid choice for~\eqref{eq:ODEFlow} with minimum kinetic energy and correct time marginals, such that $\hat{x}(t) \equiv \hat{x}(0) \sim \rho_\infty$ for all~$t$, but this precisely misses the probability current that the actual SDE displays in its stationary state.

\begin{figure}
\centering
\includegraphics[width=1.0\columnwidth]{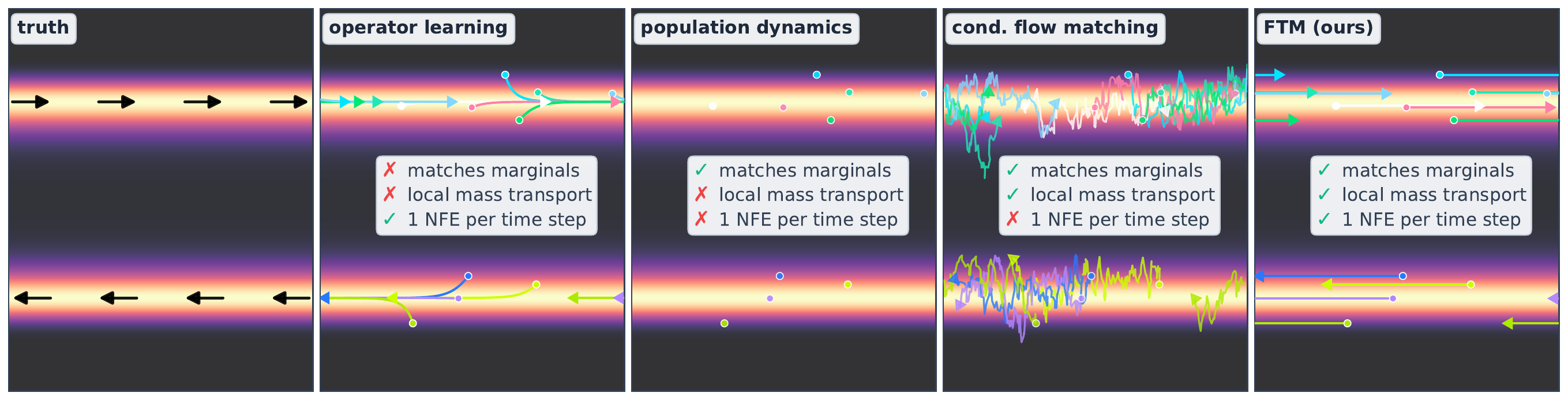}
\caption{Statistically stationary ``two-lane process'' defined in \Cref{example:two-lane}.}
\label{fig:two-lane}
\end{figure}

 We provide a second illustrative example, similar to the rotating Ornstein--Uhlenbeck process~\eqref{eq:ou-process-specs} but on the domain $S^1 \times \mathbb{R}$ (which is not simply connected) in \Cref{example:two-lane} and \Cref{fig:two-lane}.

\begin{example}
\label{example:two-lane}
Consider the following statistically stationary ``two-lane process'' $(X(t))_{t \in [0,T]}$ on the infinite cylinder $S^1 \times \mathbb{R}$:
\begin{align} \label{eq:two-lane}
\mathrm{d} X(t) =
    \begin{pmatrix}
        w(X_2(t)) \\ -U'(X_2(t))
    \end{pmatrix} \mathrm{d}t + \begin{pmatrix}
        \sqrt{2 D_1} \, \mathrm{d}W_1(t) \\ \sqrt{2 D_2} \, \mathrm{d}W_2(t)
    \end{pmatrix}\,.
\end{align}
Here, $U \colon \mathbb{R} \to [0,\infty)$ is a symmetric double well potential, concretely a quartic one $U(x_2) = \tfrac{\kappa}{4} (x_2^2 - a^2)^2$. The stationary distribution $\rho_\infty$ of~\eqref{eq:two-lane} is then uniform in $x_1$ and bimodal in $x_2$: $\rho_\infty(x_1,x_2) = Z \exp (-U(x_2)/D_2)$. We again initialize $X(0) \sim \rho(0) = \rho_\infty$. The function $w \colon \mathbb{R} \to \mathbb{R}$ in~\eqref{eq:two-lane} induces a horizontal probability current according to $v(x_1,x_2) = (w(x_2),0)^\top$. We choose~$w$ so that the current points to the right in the top well of~$U$ and to the left in the bottom well of~$U$, concretely $w(x_2) = w_0 \tanh(x_2 / l_0)$. Sample paths of the SDE~\eqref{eq:two-lane} in the statistically stationary state then behave as two ``lanes'' moving in opposite directions, with rare vertical transitions depending on the barrier height. The behavior of different methods in this example is illustrated in \Cref{fig:two-lane} and similar to~\Cref{fig:overview} and its discussion above. Note in particular that here, operator learning produces the correct horizontal movement but vertically collapses onto the modes of the two lanes, and that FTM does not show mode collapse, but cannot capture rare transitions of the true sample paths between the two wells. Numerical parameters in \Cref{fig:two-lane} for concreteness: $\kappa = 2$, $a = 1$, $D_1 = 0.012$, $D_2 = 0.12$, $w_0 = 1.25$, $l_0 = 0.18$, $T = 3$, $x_1 \in [-L,L]/\sim$ for $L = 4$.
\end{example}

\section{Further properties of the probability current velocity}\label{appendix:prob-current}

Given the examples and discussion in \Cref{appendix:IntroExample}, one can ask in which situations the probability current velocity $v$ corresponds to the minimum kinetic energy velocity field that generates the correct time marginals via the ODE~\eqref{eq:ODEFlow}. We have the following proposition:

\begin{proposition}
Assume gradient drift $b(t,x) = -\nabla U(t,x)$, and that the diffusion matrix is isotropic, $\Sigma(t) = \sigma(t)^2 I_d$.
Then:
\begin{enumerate}
    \item the current velocity \(v\) is a gradient field,
    \begin{equation*}
        v(t,x)
        = -\nabla U(t,x) - \frac{\sigma(t)^2}{2}\nabla \log \rho(t,x)
        = -\nabla\!\left(U(t,x) + \frac{\sigma(t)^2}{2}\log \rho(t,x)\right)\,.
    \end{equation*}
    \item \(v\) is the unique minimizer of the kinetic energy functional
    \begin{equation*}
        \mathcal{K}(u)
        :=
        \int_0^T \int_{\mathbb{R}^d} \rho(t,x)\,\|u(t,x)\|^2_2 \,\,\mathrm{d} x\,\,\mathrm{d} t,
    \end{equation*}
    among all velocity fields \(u\) satisfying the continuity equation
    \begin{equation*}
        \partial_t \rho + \nabla \cdot (\rho u) = 0.
    \end{equation*}
\end{enumerate}
Hence, in this setting, the probability current velocity \(v\) coincides with the minimal kinetic energy velocity field associated with the prescribed marginal path \((\rho(t))_{t\in[0,T]}\).
\end{proposition}

\begin{proof}
Under the assumptions \(b=-\nabla U\) and \(\Sigma=\sigma^2 I_d\), we compute
\begin{align*}
    v(t,x)
    &= b(t,x) - \frac12 \Sigma(t)\nabla \log \rho(t,x) = -\nabla U(t,x) - \frac{\sigma(t)^2}{2}\nabla \log \rho(t,x) \\
    &= -\nabla\!\left(U(t,x) + \frac{\sigma(t)^2}{2}\log \rho(t,x)\right).
\end{align*}
Thus \(v\) is indeed a gradient field.

Next, \(v\) is admissible for the marginal-matching problem because it satisfies
\begin{equation*}
    \partial_t \rho + \nabla\cdot(\rho v)=0.
\end{equation*}
We now show that among all admissible vector fields, there is at most one admissible gradient field. Suppose
\[
u_1 = \nabla \phi_1,
\qquad
u_2 = \nabla \phi_2
\]
both satisfy
\begin{equation*}
    \partial_t \rho + \nabla\cdot(\rho u_i)=0,
    \qquad i=1,2.
\end{equation*}
Subtracting the two equations gives
\begin{equation*}
    \nabla\cdot\bigl(\rho (u_1-u_2)\bigr)=0.
\end{equation*}
Let
\[
w := u_1-u_2 = \nabla \psi,
\qquad
\psi := \phi_1-\phi_2.
\]
Then
\begin{equation*}
    \nabla\cdot(\rho \nabla \psi)=0.
\end{equation*}
Assuming standard decay at infinity (or, alternatively, no-flux boundary conditions on a bounded domain), we multiply by \(\psi\) and integrate:
\begin{align*}
    0
    = \int_{\mathbb{R}^d} \psi \,\nabla\cdot(\rho \nabla \psi)\,\,\mathrm{d} x
    = - \int_{\mathbb{R}^d} \rho\, \lVert \nabla \psi \rVert_2^2 \,\,\mathrm{d} x.
\end{align*}
Hence
\[
\nabla \psi = 0
\qquad \text{a.e. on } \{\rho>0\},
\]
and therefore
\[
u_1=u_2
\qquad \text{a.e. with respect to } \rho(t,x)\,\,\mathrm{d} x\,\,\mathrm{d} t.
\]
So there is at most one admissible gradient field.

Finally, the minimizer of the kinetic energy functional
\[
\mathcal{K}(u)=\int_0^T \int_{\mathbb{R}^d} \rho\, \|u\|^2_2 \,\,\mathrm{d} x\,\,\mathrm{d} t
\]
subject to
\[
\partial_t\rho+\nabla\cdot(\rho u)=0
\]
is known to be a gradient field. Since \(v\) is both admissible and a gradient field, and such a field is unique, it follows that \(v\) must coincide with the unique minimizer. Therefore \(v\) is exactly the minimal kinetic energy velocity field.
\end{proof}

\begin{remark}
In general, one should \emph{not} claim that \(v\) is smoother, or has a smaller Lipschitz constant, than the minimal kinetic energy field. In the isotropic gradient-drift setting, the correct statement is stronger: $v = u_{\mathrm{min}}$
where \(u_{\mathrm{min}}\) denotes the minimal kinetic energy admissible velocity field. The equality holds because \(v\) is the unique admissible gradient field, not because it enjoys any a priori superior smoothness.
\end{remark}

\begin{remark}
If \(\Sigma(t)\) is a general anisotropic positive definite matrix, then even if
\[
b(t,x) = -\nabla U(t,x),
\]
the current velocity
\[
v(t,x)=b(t,x)-\frac12 \Sigma(t)\nabla \log \rho(t,x)
\]
need not be a Euclidean gradient field. Thus the implication
\[
\text{gradient drift} \quad \Longrightarrow \quad \text{gradient current velocity}
\]
is automatic only in the isotropic case \(\Sigma(t)=\sigma(t)^2 I_d\), or in special commutative situations. A natural anisotropic gradient-flow structure appears, however, when
\[
b(t,x) = -\Sigma(t)\nabla U(t,x).
\]
Then
\begin{equation*}
    v(t,x)
    = -\Sigma(t)\nabla\!\left(U(t,x)+\frac12 \log \rho(t,x)\right),
\end{equation*}
and \(v\) is the minimizer of the \(\Sigma(t)^{-1}\)-weighted kinetic energy
\begin{equation*}
    \int_0^T \int_{\mathbb{R}^d} \rho(t,x)\, \langle u(t,x),  \Sigma(t)^{-1} u(t,x) \rangle\,\,\mathrm{d} x\,\,\mathrm{d} t
\end{equation*}
subject to the same continuity equation constraint.
\end{remark}

\begin{remark}
Therefore, if one studies a pure overdamped Langevin diffusion with isotropic noise,
\[
\,\mathrm{d} X(t) = -\nabla U(t,X(t))\,\,\mathrm{d} t + \sigma(t)\,\,\mathrm{d} W(t),
\]
then the current velocity field \(v\) and the minimal kinetic energy velocity field are the same object. In such an example, there is no principled distinction to visualize between the two. To obtain a visible difference, one needs a genuinely irreversible example, for instance by adding a rotational component to a gradient drift as in the examples of \Cref{appendix:IntroExample}.
\end{remark}

\section{Proofs}

\subsection{Proof of Stratonovich integral identity, and loss derivation}\label{appendix:StratDotProduct}
Here, for completeness, we show that the well-known identity~\eqref{eq:StratExp} used in the main text holds. In order to show that
\begin{align*}
    \mathbb{E} \left[ \int_0^T v_\theta(t, X(t)) \circ \mathrm{d} X(t) \right] = \mathbb{E} \left[ \int_0^T \langle v_\theta(t, X(t)), v(t, X(t)) \rangle  \mathrm{d} t \right]\,,
\end{align*}
we first convert the Stratonovich integral on the left-hand side to an It{\^o} integral, such that
\begin{align*}
&\int_0^T v_\theta(t, X(t)) \circ  \mathrm{d} X(t) = \int_0^T \langle v_\theta(t, X(t)),  \mathrm{d} X(t) \rangle + \frac{1}{2} \int_0^T \Tr \left[\Sigma(t) \nabla v_\theta(t, X(t)) \right]  \mathrm{d} t \\
&= \int_0^T \langle v_\theta(t, X(t)), b(t, X(t)) \rangle  \mathrm{d} t + \int_0^T \langle v_\theta(t, X(t)), A(t)  \mathrm{d} W(t) \rangle + \frac{1}{2} \int_0^T \Tr \left[\Sigma(t) \nabla v_\theta(t, X(t)) \right]  \mathrm{d} t\,.
\end{align*}
Taking the expectation over $X$ on both sides then removes the second term on the right-hand side, since the It{\^o} integral is non-anticipating. Finally, an integration by parts, supposing boundary terms are $0$, yields
\begin{align*}
&\mathbb{E} \left[\int_0^T \Tr \left[\Sigma(t) \nabla v_\theta(t, X(t)) \right] \mathrm{d} t \right] = \int_0^T \int_{\mathbb{R}^d} \Tr \left[\Sigma(t) \nabla v_\theta(t, x) \right] \rho(t, x) \mathrm{d} x \mathrm{d} t \\
&= \sum_{i,j=1}^d \int_0^T \int_{\mathbb{R}^d} \Sigma_{ij}(t) \partial_i v_{\theta,j}(t, x) \rho(t, x) \mathrm{d} x \mathrm{d} t = -\sum_{i,j=1}^d \int_0^T \int_{\mathbb{R}^d} \Sigma_{ij}(t) v_{\theta,j}(t, x) \underbrace{\partial_i \rho(t, x)}_{=\rho \partial_i \log \rho} \mathrm{d} x \mathrm{d} t\\
&= -\mathbb{E} \left[\int_0^T \langle v_\theta(t, X(t)), \Sigma(t) \nabla \log \rho(t, X(t)) \rangle \mathrm{d} t \right]\,,
\end{align*}
which completes the proof.\\

We can hence write the objective~\eqref{eq:Loss:Intermediate01} for the velocity field to be learned as
\begin{align*}
J_T(\theta) &= \mathbb{E}_{t \sim \mathcal{U}([0, T]), X(t) \sim \rho(t)} \left[\|v_{\theta}(t, X(t))\|^2_2 - 2 \langle v(t, X(t)),  v_{\theta}(t, X(t)) \rangle\right]\\
&= \frac{1}{T} \left( \int_0^T \mathbb{E}_{X(t) \sim \rho(t)} \left[ \lVert v_\theta(t, X(t)) \rVert_2^2 - 2 \langle v(t, X(t)), v_\theta(t, X(t))  \rangle  \right] \mathrm{d} t \right) \\
&= \frac{1}{T}\mathbb{E}_{\omega \sim \mathbb{P}} \left[ \int_0^T \lVert v_\theta(t, X_\omega(t)) \rVert_2^2 \mathrm{d}t - 2 \int_0^T v_\theta(t, X_\omega(t)) \circ \mathrm{d}X_\omega(t) \right]\,,
\end{align*}
and consider discretizations of the Stratonovich integral to obtain tractable estimators given trajectory data~\eqref{eq:TrainingData}. For example, using the composite trapezoidal rule with step size $h>0$ (such that $T/h = K$ and $t_j = j h$, $j = 0, \dots, K$) for the Stratonovich integral, and a left Riemann sum for the squared velocity field, we get
\begin{align*}
J_{h,T}(\theta) &= \frac{1}{T} \mathbb{E}_{\omega \sim \mathbb{P}} \bigg[ \sum_{j = 1}^{K-1} \left \lVert v_\theta \left(t_j, X_\omega(t_j) \right) \right \rVert_2^2 h \\
&\qquad \qquad \qquad \quad - 2 \frac{v_\theta\left( t_{j+1}, X_\omega(t_{j+1})\right) + v_\theta\left(t_j, X_\omega(t_j)\right)}{2} \left(X_\omega(t_{j+1}) - X_\omega(t_j)\right) \bigg] \,.
\end{align*}

\subsection{Proof of \Cref{prop:msebound}}
\label{appendix:ProofOfMSEBound}
We first want to remark here that the stated variance bound in~\eqref{eq:MSEBound} of the main text will often be overly pessimistic because of the last step in~\eqref{eq:var-upper} below; since we deal with time averages, one could use ergodicity and mixing results to get much tighter variance bounds. We do not pursue this here since our simple, uniform-in-time bounds are enough to show that the variance of the objective $\widehat{J}_{h,\tau}(\theta)$ remains bounded as $h \downarrow 0$ at fixed $N$, in contrast to the symmetric difference-based objective from \Cref{prop:mse-sym-diff} below.

\begin{proof}[Proof of \Cref{prop:msebound}]
We split the mean squared error into variance plus bias squared, i.e.,
\begin{align}
\mathbb{E} \left[\left(\widehat{J}_{h,\tau}(\theta) - J_\tau(\theta) \right)^2 \right] = \underbrace{\mathbb{E} \left[\left(\widehat{J}_{h,\tau}(\theta) - J_{h,\tau}(\theta) \right)^2 \right]}_{\leq \left(2 V^4 + 24 \left( V^2 b_{\text{max}}^2 + \frac{\sigma_{\text{max}} V^2}{\tau} + C_{\text{div}}^2 \right) + 2 \tilde{C} h \right) \frac{1}{N} } + \underbrace{\left(J_{h,\tau}(\theta) - J_\tau(\theta) \right)^2}_{\leq 4 C^2 V^2 h^2}\,,
\label{eq:mse-strato}
\end{align}
where the indicated inequalities are what we want to show here. The bias arises from the discretization of the continuous time integrals.  Under sufficient regularity of the drift $b$, diffusion matrix $A$, and velocity field $v_\theta$,
the trapezoidal approximation of the Stratonovich integral converges with weak order $O(h)$~\cite{kloeden-platen:1992}. Therefore, there exists a constant $C > 0$ such that the bias is bounded by $2 C V h$ (note that the factor of $2$ is introduced here by convention, to align with \Cref{prop:mse-sym-diff} below), and squaring this yields the quoted upper bound. The variance is
\begin{align}
\mathbb{E} \left[\left(\widehat{J}_{h,\tau}(\theta) - J_{h,\tau}(\theta) \right)^2 \right] = \frac{1}{N} \operatorname{Var}\left( \widehat{J}^{(1)}_{h,\tau}(\theta) \right) \leq \frac{1}{N} \mathbb{E} \left[ \left( \widehat{J}^{(1)}_{h,\tau}(\theta) \right)^2 \right]\,.
\label{eq:var-upper}
\end{align}
We want to approximate the right-hand side with the continuous-time form of the objective for simplicity. The error that we commit by using the continuous approximation
\begin{align*}
\widehat{J}^{(1)}_\tau(\theta;t_1) := \frac{1}{\tau} \int_{t_1}^{t_1+\tau} \lVert v_\theta(s, X_\omega^{(1)}(s)) \rVert_2^2 \mathrm{d}s - \frac{2}{\tau}  \int_{t_1}^{t_1+\tau}  v_\theta(s, X_\omega^{(1)}(s)) \circ \mathrm{d} X_\omega(s)\,,
\end{align*}
can be written as $\widehat{J}^{(1)}_{h,\tau}(\theta) = \widehat{J}_\tau^{(1)}(\theta;t_1) + E_h$, with mean squared approximation error of the trapezoidal rule $\mathbb{E}[E_h^2] \leq \tilde{C} h$ (strong convergence order $1/2$ in general~\cite{kloeden-platen:1992}; here, again, $\tilde{C}$ depends on $b,A,v_\theta$ and derivative bounds; and we could also get an $O(h^2)$ bound e.g.\ for commutative noise). Since $(a+b)^2 \leq 2a^2 + 2b^2$, we have $\mathbb{E}\left[ \left( \widehat{J}^{(1)}_{h,\tau}(\theta) \right)^2 \right] \leq 2 \mathbb{E}\left[ \left( \widehat{J}^{(1)}_\tau(\theta;t_1) \right)^2 \right] + 2 \tilde{C} h$. Then it remains to estimate the second moment of the continuous objective (again with $(a-b)^2 \leq 2 a^2 + 2 b^2$):
\begin{align*}
\mathbb{E} \left[ \left(\widehat{J}^{(1)}_\tau(\theta) \right)^2 \right] &\leq \underbrace{2 \mathbb{E}_{t \sim {\cal U}([0, T - \tau]),X(t) \sim \rho(t)} \left[ \left(\frac{1}{\tau} \int_{t}^{t+\tau} \lVert v_\theta(s, X(s)) \rVert_2^2 \mathrm{d}s \right)^2 \right]}_{\leq 2 V^4} \\
& \quad+ 8 \mathbb{E}_{t \sim {\cal U}([0, T - \tau]), \omega \sim \mathbb{P}} \left[ \left(\frac{1}{\tau}\int_{t}^{t+\tau} v_\theta(s, X_\omega(s)) \circ \mathrm{d} X_\omega(s) \right)^2 \right]\,.
\end{align*}
To bound the expectation of the squared Stratonovich integral, fix $t$ in the following, i.e., use the law of total expectation. Convert the Stratonovich integral back to an It{\^o} integral:
\begin{align*}
&\frac{1}{\tau} \int_{t}^{t+\tau} v_\theta(s, X_\omega(s)) \circ \mathrm{d} X_\omega(s) = \frac{1}{\tau}\int_{t}^{t+\tau} \langle v_\theta(s, X_\omega(s)), b(s, X_\omega(s)) \rangle \mathrm{d} s \\
&\qquad \qquad \quad + \frac{1}{\tau}\int_{t}^{t+\tau} \langle v_\theta(s, X_\omega(s)), A(s) \mathrm{d} W_\omega(s) \rangle
+ \frac{1}{2\tau} \int_{t}^{t+\tau} \text{Tr}\left[\Sigma(s) \nabla v_\theta(s, X_\omega(s)) \right] \mathrm{d} s\,.
\end{align*}
Using the inequality $(a+b+c)^2 \leq 3(a^2 + b^2 + c^2)$ (by Cauchy-Schwarz), we bound the expectations over $\omega$ of the individual terms squared:
\begin{enumerate}
    \item Using the Cauchy-Schwarz inequality for integrals and the uniform bounds $\lVert b(t,x) \rVert_2 \leq b_{\text{max}}$ and $\lVert v_\theta(t,x) \rVert_2 \leq V$, we get
\begin{align*}
\mathbb{E}_\omega \left[ \left( \frac{1}{\tau} \int_{t}^{t+\tau}  \langle v_\theta(s, X_\omega(s)), b(s, X_\omega(s)) \rangle \mathrm{d}s \right)^2 \right] \leq V^2 b_{\text{max}}^2\,.
\end{align*}
\item Using the It{\^o} isometry and the bound $\Sigma(t) = A(t)A(t)^\top \preceq \sigma_{\text{max}} I_d$, we have
\begin{align*}
&\mathbb{E}_\omega  \left[ \left( \frac{1}{\tau} \int_{t}^{t+\tau} \langle v_\theta(s, X_\omega(s)), A(s) \mathrm{d} W_\omega(s) \rangle \right)^2 \right]\\
&= \frac{1}{\tau^2} \mathbb{E}_\omega  \left[ \int_{t}^{t+\tau} \langle v_\theta(s, X_\omega(s)),  \Sigma(s) v_\theta(s, X_\omega(s)) \rangle \mathrm{d}s \right] \\
&\leq \frac{1}{\tau^2} \mathbb{E}_\omega  \left[ \int_{t}^{t+\tau} \sigma_{\max} \lVert v_\theta(s, X_\omega(s)) \rVert_2^2 \mathrm{d}s \right] \leq \frac{\sigma_{\max} V^2}{\tau}\,.
\end{align*}
\item Lastly, using the assumption $C_{\text{div}} = \sup_{t,x} \frac{1}{2} \lvert \text{Tr}\left[\Sigma(t) \nabla v_\theta(t,x) \right] \rvert < \infty$, we have 
\begin{align*}
\mathbb{E}_\omega  \left[ \left( \frac{1}{2\tau} \int_{t}^{t+\tau} \Tr \left[\Sigma(s) \nabla v_\theta(s, X_{\omega}(s)) \right] \mathrm{d}s \right)^2 \right] \leq C_{\text{div}}^2\,,
\end{align*}
\end{enumerate}
which completes the proof since all of these bounds are independent of $ t \sim {\cal U}([0, T - \tau])$.
\end{proof}

\subsection{Variance of the one-step FTM loss as $h \downarrow 0$, and relation to symmetric differences}\label{appendix:SymEstimatorUnstable}

In this section, we discuss the one-step FTM objective~\eqref{eq:FTM:OneStepEstimator} from the main text. We first note that the loss $J_{\text{FTM}}(\theta)$, prior to empiricalization, can be written as
\begin{align}
    J_{\text{FTM}}(\theta) &= h^{-1} \mathbb{E}_{t \sim \mathcal{U}([h, T-h]), \omega \sim \mathbb{P}} \left[\|v_{\theta}(t, X_\omega(t))\|_2^2 h - \langle v_{\theta}(t,X_\omega(t)), X_\omega(t+h) - X_\omega(t-h)\rangle \right] \nonumber\\
    &= \mathbb{E}_{t \sim \mathcal{U}([h, T-h]), \omega \sim \mathbb{P}} \left[\|v_{\theta}(t, X_\omega(t))\|^2_2 - 2 \langle Y_{h,\omega}(t) , v_{\theta}(t, X_\omega(t))\rangle\right]\,,
    \label{eq:FTM-via-sym-diff}
\end{align}
where we denote the symmetric finite difference of the SDE paths with step size $h$ by
\begin{equation*}
  Y_{h,\omega}(t) := \frac{X_\omega(t+h) - X_\omega(t-h)}{2h}\,. 
\end{equation*} 
Indeed, while the one-step FTM objective~\eqref{eq:FTM:OneStepEstimator} has been obtained from~\eqref{eq:strato-chunk-obj} in the main text in the limiting case $\tau = 2h$ and midpoint discretization of the Stratonovich integral, we can see from~\eqref{eq:FTM-via-sym-diff} that there is another, in some sense complementary way of viewing~\eqref{eq:FTM:OneStepEstimator} based on the discussion in \Cref{sec:motivate-current} of the main text: In~\eqref{eq:FTM-via-sym-diff}, by the law of total expectation, we have
\begin{align*}
    \mathbb{E}_{\omega \sim \mathbb{P}} \left[ \langle Y_{h,\omega}(t) , v_{\theta}(t, X_\omega(t))\rangle\right] = \mathbb{E}_{\omega \sim \mathbb{P}} [ \langle \underbrace{\mathbb{E} \left[ Y_{h,\omega}(t) \mid X_\omega(t) \right]}_{\overset{\eqref{eq:SymIncrement}}{=}v_h(t,X_\omega(t))} , v_{\theta}(t, X_\omega(t))\rangle]\,,
\end{align*}
i.e., we recognize here the approximation of the probability current velocity via symmetric differences as in~\eqref{eq:SymIncrement}.
For $h \downarrow 0$, this term converges to $\mathbb{E}_{\omega \sim \mathbb{P}} \left[ \langle v(t, X_\omega(t)) , v_{\theta}(t, X_\omega(t))\rangle\right]$ by \Cref{sec:motivate-current}. In summary, as an alternative to the route to the FTM loss~\eqref{eq:FTM:OneStepEstimator} via the Stratonovich integral identity~\eqref{eq:StratExp} and subsequent discretization and $\tau = 2h$ as presented in the main text, we could hence have started directly from the intractable loss~\eqref{eq:Loss:Intermediate01}, approximated the true, unknown probability current velocity $v$ by the symmetric difference $v_h$ at finite $h > 0$, and then used the law of total expectation to arrive at~\eqref{eq:FTM:OneStepEstimator}. We chose to start from the Stratonovich integral identity~\eqref{eq:StratExp} in the main text to closely connect the loss to the available literature in statistical physics and stochastic thermodynamics~\cite{frishman-ronceray:2020,boffi-vanden-eijnden:2025,lyu-ray-crutchfield:2025}, and since chunking -- which prevents the variance blowup as $h \downarrow 0$ at fixed $N$ -- is somewhat more natural in this setting. In hindsight, then, it is easy to realize that \textit{averaging} the symmetric difference estimator over $K_\tau$ time steps is equivalent, up to negligible boundary terms for $K_\tau \gg 1$, to the chunked estimator~\eqref{eq:FTMLoss}.

Let us discuss the aforementioned variance blowup of the one-step FTM loss in more detail. While in principle, this is a special case of the ``chunked'' result in \Cref{prop:msebound}, up to the different choice of discretization rule, we prefer to present a separate analysis and further comments on it for clarity. We consider the discretized objective
\begin{equation*}
    J_h(\theta) = \mathbb{E}_{t \sim \mathcal{U}([h, T-h]), \omega \sim \mathbb{P}} \left[\|v_{\theta}(t, X_\omega(t))\|^2_2 - 2 \langle Y_{h,\omega}(t) , v_{\theta}(t, X_\omega(t))\rangle\right]\,,
\end{equation*}
and the corresponding empirical estimator
\begin{equation}\label{eq:sym-diff-empirical}
\widehat{J}_h(\theta) = \frac{1}{N}\sum_{i = 1}^N \|v_{\theta}(t_i, X^{(i)}(t_i))\|_2^2 - 2 \langle Y^{(i)}_h(t_i), v_{\theta}(t_i, X^{(i)}(t_i)\rangle\,,
\end{equation}
which provides a practically tractable objective to replace $J(\theta)$ from~\eqref{eq:Loss:Intermediate01} in the main text.
In the following, we will make the two assumptions:
\begin{enumerate}
    \item For fixed $\theta \in \mathbb{R}^p$, the velocity field $v_{\theta}$ is bounded in the sense that there exists $V > 0$ such that 
\begin{equation*}
\|v_{\theta}(t, x)\|_2 \leq V \quad \text{ for all } t \in [0,T], \; x \in \mathbb{R}^d\,.
\end{equation*}
    \item The symmetric increment \eqref{eq:SymIncrement} is first-order accurate in the sense
\begin{equation} \label{eq:AsmFirstOrderAccurate}
\|v_h(t, x) - v(t, x)\|_2 \leq C h\,,
\end{equation}
for a constant $C > 0$ and all $t \in [h, T - h]$, $x \in \mathbb{R}^d$. We formulate this as an assumption here for brevity, but such a bound can be derived by using standard weak Taylor arguments for the forward and reverse process~\cite{kloeden-platen:1992}.
\end{enumerate}
We have the following upper bound on the mean-squared error of~\eqref{eq:sym-diff-empirical} under these assumptions:

\begin{proposition} \label{prop:mse-sym-diff}
We have
\begin{equation}\label{eq:MSEBound-symdiff}
\mathbb{E}\left[(\widehat{J}_h(\theta) - J(\theta))^2\right] \leq 4 C^2 V^2 h^2 + 6 V^2 d \sigma_{\max} (N h)^{-1} + \left(3V^4 + 6 b_{\max}^2 V^2 \right) N^{-1}\,.
\end{equation}
\end{proposition}
The first term on the right-hand side of~\eqref{eq:MSEBound-symdiff} bounds the squared bias $(J_{h}(\theta) - J(\theta))^2 \lesssim h^2$, and the other terms bound the variance $\mathbb{E} \left[(\widehat{J}_{h}(\theta) - J_{h}(\theta))^2 \right] \lesssim (Nh)^{-1}$ for small $h$, which diverges at fixed $N$ as $h \downarrow 0$. This suggests that as the time step size $h$ decreases, more and more samples $N$ are necessary to keep the mean-squared error bounded. A possible choice is to balance the bias and variance by picking $h^2 \approx (Nh)^{-1}$, or $N \approx 1 / h^3$ (ignoring constants). While \Cref{prop:mse-sym-diff} a priori only gives an upper bound on the mean-squared error, a matching lower bound $\gtrsim (Nh)^{-1}$ can be shown, e.g., when assuming uniform ellipticity $\Sigma(t) \succeq \sigma_{\min}I_d$ for all $t \in [0,T]$ and $\sigma_{\min} > 0$. We find it more instructive to instead present an explicit calculation of the mean-squared error of the symmetric finite difference objective for the special case of $(X(t))_{t \in [0,T]}$ a one-dimensional Brownian motion, which we formulate as a lemma.

\begin{lemma}\label{lemma:1d-bm-error}
    Consider the $(d = 1)$-dimensional SDE $\mathrm{d}X(t) = \mathrm{d} W(t)$, $X(0) \sim {\cal N}(0,1)$, on the time interval $t \in [0,1]$. Then we have the sharp asymptotic equivalence $\mathbb{E}\left[(\widehat{J}_h(\theta) - J(\theta))^2\right] \sim c (Nh)^{-1}$ as $h \downarrow 0$ for a constant $c > 0$ (explicitly given by $c = 2 \int_0^1 \mathbb{E}_{Z \sim {\cal N}(0,1)}\left[v_\theta(t, \sqrt{1 + t} Z)^2 \right] \mathrm{d}t$ here).
\end{lemma}

Intuitively, then, since at small $h$ only the short-time behavior of
the diffusion process matters also for more general SDEs driven by Brownian
noise, this will remain true in higher dimensions and with nonlinear drift
vector field and anisotropic diffusion matrix. Finally, as discussed further in the main text, we remark that constants do matter in practice, in the sense that the time step $h$ in many applications may not actually be so small as to see the blow-up of the variance at fixed $N$ as $h \downarrow 0$ of~\eqref{eq:sym-diff-empirical}. In such cases, the objective~\eqref{eq:sym-diff-empirical} does provide a viable option for learning the probability current velocity.

\begin{proof}[Proof of \Cref{prop:mse-sym-diff}]
    Decompose the difference $\widehat{J}_h(\theta) - J(\theta)$ as
\begin{equation}\label{eq:ProofIntermedI}
\widehat{J}_h(\theta) - J(\theta) = (\widehat{J}_h(\theta) - J_h(\theta)) + (J_h(\theta) - J(\theta))\,.
\end{equation}
The first term in \eqref{eq:ProofIntermedI} has mean zero because $\widehat{J}_h(\theta)$ is a Monte Carlo estimator of $J_h(\theta)$ with independent samples. We now use that $\mathbb{E}[(Z_1 + Z_2)^2] = \mathbb{E}[Z_1^2] + 2\mathbb{E}[Z_1Z_2] + \mathbb{E}[Z_2^2]$ and identify $Z_1$ with the first term of \eqref{eq:ProofIntermedI} and $Z_2$ with the second term of \eqref{eq:ProofIntermedI}. Notice that $Z_2$ (the second term of \eqref{eq:ProofIntermedI}) is deterministic so that
\begin{equation}\label{eq:Proof:SimplifiedEVarBiasTerm}
\mathbb{E}\left[\left(\widehat{J}_h(\theta) - J(\theta)\right)^2\right] = \mathbb{E}\left[(\widehat{J}_h(\theta) - J_h(\theta))^2\right] + (J_h(\theta) - J(\theta))^2\,.
\end{equation}
Let us first consider the squared bias term $(J_h(\theta) - J(\theta))^2$ in \eqref{eq:Proof:SimplifiedEVarBiasTerm}. With the linearity of the expectation, we obtain
\[
J_h(\theta) - J(\theta) = -2\mathbb{E}_{t \sim \mathcal{U}([h, T - h]), \omega \sim \mathbb{P}}\left[\left\langle Y_{h,\omega}(t) - v(t, X_{\omega}(t)), v_{\theta}(t, X_{\omega}(t)\right\rangle\right]\,.
\]
With the law of total probability, we obtain
\begin{align*}
\mathbb{E}\left[\left\langle Y_{h,\omega}(t) - v(t, X_{\omega}(t)), v_{\theta}(t, X_{\omega}(t)\right\rangle\right] = & \mathbb{E}\left[\mathbb{E}\left[\left\langle Y_{h,\omega}(t) - v(t, X_{\omega}(t)), v_{\theta}(t, X_{\omega}(t)\right\rangle | X_{\omega}(t)\right]\right]\\
= & \mathbb{E}\left[\left\langle \mathbb{E}\left[Y_{h,\omega}(t) - v(t, X_{\omega}(t)) | X_{\omega}(t)\right], v_{\theta}(t, X_{\omega}(t))\right\rangle\right]\,.
\end{align*}
Now use \eqref{eq:AsmFirstOrderAccurate} and the boundedness of $v_{\theta}$ to obtain with the Cauchy-Schwarz inequality
\[
|J_h(\theta) - J(\theta)| \leq 2 \mathbb{E}\left[\|\mathbb{E}[Y_{h, \omega}(t) - v(t, X_{\omega}(t)) | X_{\omega}(t)]\|_2 \|v_{\theta}(t, X_{\omega}(t))\|_2\right] \leq 2 C h V\,.
\]
Let us now bound the first term of \eqref{eq:Proof:SimplifiedEVarBiasTerm}. Because $\widehat{J}_h(\theta)$ is an unbiased estimator of $J_h(\theta)$ and all samples in the estimator are independent and identically distributed, we obtain
\[
\mathbb{E}\left[\left(\widehat{J}_h(\theta) - J_h(\theta)\right)^2\right] = \frac{1}{N}\operatorname{Var}_{t \in \mathcal{U}([h, T - h]), \omega \sim \mathbb{P}}\left[\|v_{\theta}(t, X_{\omega}(t))\|_2^2 - 2\langle Y_{h,\omega}(t), v_{\theta}(t, X_{\omega}(t))\rangle\right]\,.
\]
It will be convenient to decompose the random variable
\begin{align}
&\|v_{\theta}(t, X_{\omega}(t))\|_2^2 - 2 \langle Y_{h,\omega}(t), v_{\theta}(t, X_{\omega}(t))\rangle \nonumber\\
= &\|v_{\theta}(t, X_{\omega}(t))\|_2^2 - 2\langle D{h,\omega}(t), v_{\theta}(t, X_{\omega}(t))\rangle - 2\langle \eta_{h,\omega}(t), v_{\theta}(t, X_{\omega}(t))\rangle\,,\label{eq:Proof:OneSampleZ}
\end{align}
where we decomposed $Y_{h,\omega}(t)$ into 
\begin{align*}
Y_{h,\omega}(t) = D_{h,\omega}(t) + \eta_{h,\omega}(t)\,,\quad  \text{with } \begin{cases} D_{h,\omega}(t) &= \frac{1}{2h}\int_{t - h}^{t + h} b(s, X_{\omega}(s))\mathrm ds\,, \\\eta_{h,\omega}(t) &= \frac{1}{2h} \int_{t - h}^{t + h} A(s) \mathrm dW_\omega(s)\,.
\end{cases}
\end{align*}
Consider the term $\langle \eta_{h,\omega}(t), v_{\theta}(t, X_{\omega}(t))\rangle$ in \eqref{eq:Proof:OneSampleZ} first. For a fixed $h > 0$, this term is square integrable because the Cauchy-Schwarz inequality gives
\[
|\langle\eta_{h,\omega}(t), v_{\theta}(t, X_{\omega}(t))\rangle|^2 \leq \|\eta_{h,\omega}(t)\|^2 \|v_{\theta}(t, X_{\omega}(t))\|^2\,,
\]
and with the boundedness of $v_\theta$ and diffusion matrix $\Sigma(t) \preceq \sigma_{\max} I_d$, we obtain via It{\^o} isometry:
\[
\mathbb{E}[|\langle\eta_{h,\omega}(t), v_{\theta}(t, X_{\omega}(t))\rangle|^2] \leq V^2 \frac{d \sigma_{\text{max}}}{h} < \infty\,.
\]
Correspondingly, we bound the variance as 
\begin{equation}\label{eq:Proof:Var1}
\operatorname{Var}[-2\langle\eta_{h,\omega}(t), v_{\theta}(t, X_{\omega}(t))\rangle] \leq \mathbb{E}[|2\langle\eta_{h,\omega}(t), v_{\theta}(t, X_{\omega}(t))\rangle|^2] \leq 4V^2 \frac{d \sigma_{\text{max}}}{2h}\,.
\end{equation}
Let us now consider the term $- 2\langle D_{h,\omega}(t), v_{\theta}(t, X_{\omega}(t))\rangle$ of \eqref{eq:Proof:OneSampleZ}, and obtain
\[
|- 2\langle D_{h,\omega}(t), v_{\theta}(t, X_{\omega}(t))\rangle| \leq 2 \|D_{h, \omega}(t)\|_2 \|v_{\theta}(t, X_{\omega}(t))\|_2 \leq 2b_{\text{max}} V\,,
\]
where we used that the drift term $b$ is upper bound by $b_{\text{max}}$ and $v_\theta$ by $V$. The variance is then
\begin{equation}\label{eq:Proof:Var2}
\operatorname{Var}\left[- 2\langle D_{h,\omega}(t), v_{\theta}(t, X_{\omega}(t))\rangle\right] \leq \mathbb{E}\left[|-2 \langle D_{h,\omega}(t), v_{\theta}(t, X_{\omega}(t))\rangle|^2\right] \leq 4b_{\text{max}}^2V^2\,.
\end{equation}
Analogously, we bound
\begin{equation}\label{eq:Proof:Var3}
\operatorname{Var}\left[\|v_{\theta}(t, X_{\omega}(t))\|_2^2\right] \leq \mathbb{E}\left[\|v_{\theta}(t, X_{\omega}(t))\|_2^4\right] \leq V^4\,.
\end{equation}
Combining \eqref{eq:Proof:Var1}, \eqref{eq:Proof:Var2}, and \eqref{eq:Proof:Var3} with  $\operatorname{Var}[Z_1 + Z_2 + Z_3] \leq 3(\operatorname{Var}[Z_1] + \operatorname{Var}[Z_2] + \operatorname{Var}[Z_3])$ leads to \eqref{eq:MSEBound-symdiff}.
\end{proof}

\begin{proof}[Proof of \Cref{lemma:1d-bm-error}]
    It suffices to consider the variance $\mathbb{E}\left[(\widehat{J}_h(\theta) - J_h(\theta))^2\right]$, and, by independence of the samples in~\eqref{eq:sym-diff-empirical} and the law of total variance to condition on the time $t \sim {\cal U}([h, 1 - h])$, to compute $\operatorname{Var}\widehat{J}_h(\theta; t)$ for
    \begin{align}
    \widehat{J}_h(\theta;t) := v_\theta(t, X_\omega(t))^2 - 2 v_\theta(t, X_\omega(t)) Y_{h,\omega}(t)
    \label{eq:objective-bm}
\end{align}
at a fixed time $t \in [h, 1-h]$. First, we characterize the joint distribution of $(X_\omega(t), Y_{h,\omega}(t))$. Clearly $X_\omega(t) \sim {\cal N}(0,1+t)$, and
\begin{align*}
    Y_{h,\omega}(t) &= \frac{X_\omega(t+h) - X_\omega(t-h)}{2h} = \frac{(X_\omega(t+h)-X_\omega(t)) + (X_\omega(t)-X_\omega({t-h}))}{2h} \\
    &= \frac{\Delta B_\omega(t) + \Delta B_\omega(t-h)}{2 h} \sim {\cal N}\left(0,\frac{1}{2h}\right)
\end{align*}
marginally. Their correlation is then
\begin{align*}
    \mathbb{E} \left[X_\omega(t) Y_{h,\omega}(t) \right] = \frac{1}{2h} \mathbb{E} \left[ \left(X_\omega(t-h) + \Delta B_\omega(t-h) \right) \left(\Delta B_\omega(t) + \Delta B_\omega(t-h) \right) \right] = \frac{1}{2}
\end{align*}
such that in total
\begin{align*}
    \begin{pmatrix}
        X_\omega(t)\\ Y_{h,\omega}(t)
    \end{pmatrix} \sim {\cal N} \left(0, \begin{pmatrix}
        1+t & 1/2\\
        1/2 & 1/(2h)
    \end{pmatrix} \right)\,.
\end{align*}
We can hence also write
\begin{align}
    \begin{pmatrix}
        X_\omega(t) \\ Y_{h,\omega}(t)
    \end{pmatrix} \overset{d}{=} \begin{pmatrix}
        \sqrt{1+t} Z_1 \\
        \frac{1}{2 \sqrt{1+t}} Z_1 + \sqrt{\frac{1}{2h} - \frac{1}{4(1+t)}} Z_2
    \end{pmatrix} 
    \label{eq:standard-normal}
\end{align}
for $Z_1, Z_2$ iid ${\cal N}(0,1)$. The other ingredient we need is that for any $k \geq 0$, by the law of total probability, we have 
\begin{align}
    \mathbb{E} \left[v_\theta(t, X_\omega(t))^k Y_{h,\omega}(t) \right] = \mathbb{E} \left[v_\theta(t, X_\omega(t))^k \mathbb{E} \left[ Y_{h,\omega}(t) \mid X_\omega(t) \right]\right]  = \mathbb{E} \left[v_\theta(t, X_\omega(t))^k v_h(X_\omega(t), t)\right] \label{eq:vh-identity}
\end{align}
where $\lim_{h \to 0} v_h(t,x) = v(t,x)$ is the true probability current velocity of the process $X$. In the present example
\begin{align*}
    v(t,x) = b(t,x) - \frac{1}{2} \Sigma(t) \nabla \log \rho(t,x) = \frac{x}{2(1+t)}\,.
\end{align*}
With that, we can now compute
\begin{align*}
   \operatorname{Var} \widehat{J}_h(\theta;t) = \underbrace{\operatorname{Var}\left(v_\theta(t, X_\omega(t))^2 \right)}_{=:V_1} + 4\underbrace{\operatorname{Var}\left(v_\theta(t, X_\omega(t)) Y_{h,\omega}(t) \right)}_{=:V_2} \\
   - 4 \underbrace{\operatorname{Cov}\left(v_\theta(t, X_\omega(t))^2,v_\theta(t, X_\omega(t))Y_{h,\omega}(t) \right)}_{=:C_{12}}\,.
\end{align*}
The first term $V_1$ is a constant in $h$. The last term is
\begin{align*}
    &C_{12} = \mathbb{E}\left[v_{\theta}(t, X_\omega(t))^3 Y_{h,\omega}(t) \right] - \mathbb{E}\left[v_\theta(t, X_\omega(t))^2 \right] \mathbb{E}\left[v_\theta(t, X_\omega(t)) Y_{h,\omega}(t) \right]\\
    &\overset{\eqref{eq:vh-identity}}{=}\mathbb{E}\left[v_{\theta}(t, X_\omega(t))^3 v_h(t, X_\omega(t)) \right] - \mathbb{E}\left[v_\theta(t, X_\omega(t))^2 \right] \mathbb{E}\left[v_\theta(t, X_\omega(t)) v_h(t, X_\omega(t)) \right]\,,
\end{align*}
which evidently converges to a well-defined limiting value as $h \to 0$. Lastly, for $V_2$, we find
\begin{align*}
    V_2 &= \mathbb{E} \left[v_\theta(t, X_\omega(t))^2 \left(Y_{h,\omega}(t) \right)^2 \right] - \mathbb{E} \left[ v_\theta(t, X_\omega(t)) Y_{h,\omega}(t) \right]^2\\
    &\overset{\eqref{eq:standard-normal},\eqref{eq:vh-identity}}{=} \frac{1}{4(1+t)} \mathbb{E} \left[v_\theta(t, \sqrt{1+t} Z_1)^2 Z_1^2 \right] + \left(\frac{1}{2 h} - \frac{1}{4(1+t)} \right) \mathbb{E} \left[v_\theta(t, \sqrt{1+t} Z_1)^2 \right]\\
    & \qquad  - \mathbb{E} \left[v_\theta(t, X_\omega(t)) v_h(t, X_\omega(t)) \right]^2\,.
\end{align*}
Overall, because of the term $1/(2h)$ term, we hence have shown
\begin{align*}
   \operatorname{Var}\widehat{J}_h(\theta;t) \overset{h \to 0}{\sim} \frac{c(t)}{h}\,,
\end{align*}
where the constant $c(t)>0$ is explicitly given by $c(t) = 2\mathbb{E} \left[v_\theta(t, \sqrt{1+t} Z)^2 \right]$ for $Z \sim {\cal N}(0,1)$.
\end{proof}

\subsection{Proof of \Cref{prop:w2}}\label{appendix:ProofOfW2}

We denote in this section by $(X(t))_{t \in [0,T]}$ sample paths of the SDE~\eqref{eq:SDE}, by $(\hat{x}_{\theta}(t))_{t \in [0,T]}$ the trajectories of the ODE~\eqref{eq:ODEFlow} with (learned) velocity field $ u = v_\theta$, and by $(\hat{x}(t))_{t \in [0,T]}$ the trajectories of the ODE~\eqref{eq:ODEFlow} with $u = v$ the true probability current velocity~\eqref{eq:CurrentVelocity}. We couple the distributions of these by picking\textit{ the same starting point} $X(0) = \hat{x}_{\theta}(0) = \hat{x}(0) \sim \rho(0)$, state pathwise errors in \Cref{lemma:exact-ode-vs-sde} and \Cref{lemma:exact-ode-vs-approx-ode} via Gr{\"o}nwall's lemma, and conclude with the Wasserstein-2 bound of \Cref{prop:w2} as a corollary. As in the main text in \Cref{prop:w2}, we suppose here that the probability current velocity $v$ is uniformly Lipschitz on $\mathbb{R}^d$ for all times with Lipschitz constant $L_v > 0$. Then we have:

\begin{lemma} \label{lemma:exact-ode-vs-sde}
    $\mathbb{E} \left[\lVert \hat{x}(t) - X(t) \rVert_2^2 \right] \leq \exp((2 L_v + 1)t) \int_0^t \left(\Tr \Sigma(s) + \mathbb{E} \left[\lVert \tfrac{1}{2} \Sigma(s) \nabla \log \rho(s, X(s)) \rVert_2^2 \right] \right) \mathrm{d}s$.
\end{lemma}

\begin{lemma} \label{lemma:exact-ode-vs-approx-ode}
   $\mathbb{E} \left[\lVert \hat{x}_{\theta}(t) - \hat{x}(t) \rVert_2^2 \right] \leq \exp((2 L_v + 1)t) \int_0^t \mathbb{E} \left[ \lVert v_{\theta}(s, \hat{x}_{\theta}(s)) - v(s, \hat{x}_{\theta}(s)) \rVert_2^2 \right] \mathrm{d}s$.
\end{lemma}

We remark that combining these two lemmas using the triangle inequality also provides an upper bound on the mean squared distance of $X(t)$ and $\hat{x}_\theta(t)$ if desired, though we come back to this problem in more detail in \Cref{prop:path-error} and \Cref{appendix:InferenceErrorRollout}. Now, since the time-marginal distributions $\rho(t)$ generated by $(\hat{x}(t))_{t \in [0,T]}$ are the true ones, and we have constructed a particular coupling of $\rho(t)$ and $\hat{\rho}_\theta(t)$ via the two ODEs with the same initial conditions, \Cref{lemma:exact-ode-vs-approx-ode} directly implies \Cref{prop:w2}:
\begin{align*}
    W_2^2(\rho(t), \hat{\rho}_\theta(t)) = \inf_{\gamma \in \Gamma(\rho(t), \hat{\rho}_\theta(t))} \mathbb{E}_{(x,y) \sim \gamma} \left[\lVert x - y \rVert_2^2 \right] \leq \mathbb{E} \left[\lVert \hat{x}_{\theta}(t) - \hat{x}(t) \rVert_2^2 \right]\,,
\end{align*}
where $\Gamma(\rho(t), \hat{\rho}_\theta(t))$ denotes the set of all couplings of $\rho(t)$ and $\hat{\rho}_\theta(t)$. Finally, then, as mentioned in the main text, if $\phi \colon \mathbb{R}^d \to \mathbb{R}$ is a scalar QoI which is Lipschitz with constant $L_\phi$, then we obtain by Kantorovich--Rubinstein duality for the $W_1$ metric, and $W_1 \leq W_2$:
\begin{align*}
    |\mathbb{E}_{X(t) \sim \rho(t)}[\phi(X(t))] - \mathbb{E}_{\hat{x}_\theta(t) \sim \hat{\rho}_\theta(t)}[\phi(\hat{x}_\theta(t))]| \leq L_\phi \; W_1(\rho(t), \hat{\rho}_\theta(t)) \leq L_\phi \; W_2(\rho(t), \hat{\rho}_\theta(t))\,,
\end{align*}
such that \Cref{prop:w2} also provides control over the error for time-local QoIs.

\begin{proof}[Proof of \Cref{lemma:exact-ode-vs-sde}]
Denote the difference of the ODE and SDE trajectory by $E(t) := \hat{x}(t) - X(t)$. By construction we have $E(0) = 0$, and $E$ satisfies the SDE
\begin{align*}
    \mathrm{d} E(t) &= v(t, \hat{x}(t)) \mathrm{d}t - b(t, X(t)) \mathrm{d}t - A(t) \mathrm{d}W(t)\\
    &\overset{\eqref{eq:CurrentVelocity}}{=} \left[v(t, \hat{x}(t)) - v(t, X(t)) - q(t) \right] \mathrm{d}t - A(t) \mathrm{d}W(t)\,,
\end{align*}
where we abbreviate $q(t) := \tfrac12 \Sigma(t) \nabla \log \rho(t, X(t))$. By It{\^o}'s lemma, we then have
\begin{align*}
    \mathrm{d} \lVert E(t) \rVert_2^2 = \left( 2 \langle E(t), v(t, \hat{x}(t)) - v(t, X(t)) - q(t) \rangle + \Tr \Sigma(t) \right) \mathrm{d}t - 2 \langle E(t), A(t) \mathrm{d} W(t) \rangle\,,
\end{align*}
and by taking expectations:
\begin{align*}
    \frac{\mathrm{d}}{\mathrm{d}t} \mathbb{E} \left[\lVert E(t) \rVert_2^2 \right] = 2 \mathbb{E} \left[\langle E(t), v(t, \hat{x}(t)) - v(t, X(t)) \rangle \right] + 2 \mathbb{E} \left[\langle E(t), q(t) \rangle \right] + \Tr \Sigma(t)\,.
\end{align*}
Using the Cauchy--Schwarz inequality and Lipschitz property of $v$ for the first term on the right-hand side and Young's inequality for the second term, we obtain
\begin{align*}
     \frac{\mathrm{d}}{\mathrm{d}t} \mathbb{E} \left[\lVert E(t) \rVert_2^2 \right] \leq (2 L_v + 1) \left[\lVert E(t) \rVert_2^2 \right] + \mathbb{E} \left[\lVert q(t) \rVert_2^2 \right] + \Tr \Sigma(t)\,,
\end{align*}
such that Gr{\"o}nwall's lemma, noting  that $\mathbb{E} \left[\lVert E(0) \rVert_2^2 \right] = 0$, finishes the proof.
\end{proof}

\begin{proof}[Proof of \Cref{lemma:exact-ode-vs-approx-ode}]
Define $e(t) = \hat{x}_{\theta}(t) - \hat{x}(t)$ with $\frac{\mathrm{d}}{\mathrm{d}t} e(t) = v_\theta(t, \hat{x}_{\theta}(t)) - v(t, \hat{x}(t)) = (v_\theta(t, \hat{x}_{\theta}(t)) - v(t, \hat{x}_{\theta}(t))) + (v(t, \hat{x}_{\theta}(t)) -  v(t, \hat{x}(t)))$ and proceed analogously to the previous proof.
\end{proof}

\subsection{Proof of \Cref{prop:path-error}}\label{appendix:InferenceErrorRollout}

For the reader's convenience, we state again all necessary assumptions for \Cref{prop:path-error} in the following. Let $\rho(t)$ denote the law of the SDE~\eqref{eq:SDE}, and let
$\hat{\rho}_\theta(t)$ denote the law generated by the learned ODE
\[
    \frac{\mathrm d}{\mathrm dt}\hat{x}_\theta(t)
    =
    v_\theta(t,\hat{x}_\theta(t)),
    \qquad
    \hat{x}_\theta(0)\sim \rho(0).
\]
Assume that the probability current velocity $v$ is uniformly Lipschitz in $x$
with constant $L_v>0$, and that we have pointwise error control
\[
    \|v_\theta(t,x)-v(t,x)\|_2 \leq \epsilon
    \qquad
    \text{for all } t\in[0,T],\ x\in\mathbb R^d .
\]
Let $\phi:\mathbb [0,T] \times \mathbb R^d\to\mathbb R$ be a bounded
and Lipschitz vector field:
\[
    \|\phi(t,x)\|_2\leq B_\phi,
    \qquad
    \|\phi(t,x)-\phi(t,y)\|_2
    \leq
    L_{\phi}\|x-y\|_2 \qquad
    \text{for all } t\in[0,T],\ x,y\in\mathbb R^d\,.
\]
Assume similarly that $v_\theta$ is uniformly bounded and Lipschitz in $x$:
\[
    \|v_\theta(t,x)\|_2\leq B_{v_\theta},
    \qquad
    \|v_\theta(t,x)-v_\theta(t,y)\|_2
    \leq
    L_{v_\theta}\|x-y\|_2 \qquad
    \text{for all } t\in[0,T],\ x,y\in\mathbb R^d\,.
\]
Set
$\tilde{C}:= L_{\phi}B_{v_\theta} + B_\phi L_{v_\theta}.
$
We will show below that the error in the path-dependent QoI $\Phi(X) := \int_0^T \phi(t,X(t)) \circ \mathrm{d} X(t)$ under the true SDE
and the learned ODE rollout satisfies
\begin{align}
&
\left|
    \mathbb E[\Phi(X) ]
    -
    \mathbb E[\Phi(\hat{x}_\theta) ]
\right|
\leq \int_0^T \left(
\|v(t,\cdot)-v_\theta(t,\cdot)\|_{L^2(\rho(t))}
\|\phi(t,\cdot)\|_{L^2(\rho(t))}
+
\tilde{C} W_2(\rho(t),\hat{\rho}_\theta(t)) \right) \mathrm{d}t\,.
\label{eq:qoi-rate-local-plus-marginal}
\end{align}
Using \Cref{prop:w2}, this estimate directly implies \Cref{prop:path-error}. We also mention here (cf.~\eqref{eq:proof-rate-decomposition} in the proof below) that the local transport term, i.e., the first term on the right-hand side of~\eqref{eq:qoi-rate-local-plus-marginal}, can alternatively be controlled directly by the FTM excess
risk
\begin{equation*}
    J(\theta)-J^*
    =
    \frac1T
    \int_0^T
    \|v_\theta(t,\cdot)-v(t,\cdot)\|_{L^2(\rho(t))}^2
    \,\mathrm dt\,,
\end{equation*}
where 
\begin{align*}
\begin{cases}
    J(\theta)
    &=
    \frac1T
    \int_0^T
    \int_{\mathbb R^d}
    \left(
        \|v_\theta(t,x)\|_2^2
        -
        2\langle v(t,x),v_\theta(t,x)\rangle
    \right)
    \rho(t,x)\,\mathrm dx\,\mathrm dt \,,\\
    J^*
    &:=
    \frac1T
    \int_0^T
    \int_{\mathbb R^d}
    \left(
        \|v(t,x)\|_2^2
        -
        2\|v(t,x)\|_2^2
    \right)
    \rho(t,x)\,\mathrm dx\,\mathrm dt
    \\
    &=
    -
    \frac1T
    \int_0^T
    \int_{\mathbb R^d}
    \|v(t,x)\|_2^2
    \rho(t,x)\,\mathrm dx\,\mathrm dt \,,
\end{cases}
\end{align*}
since by using Cauchy--Schwarz twice:
\begin{align*}
&
\frac1T
\int_0^T
\left|
    \mathbb{E}_{X(t) \sim \rho(t)} \left[
    \left\langle
        \phi(t,X(t)),v(t,X(t))-v_\theta(t,X(t))
    \right\rangle
    \right]
\right|
\,\mathrm dt
\notag \\
&\qquad\leq
\left(J(\theta)-J^*\right)^{1/2}
\left(
    \frac1T
    \int_0^T
    \|\phi(t,\cdot)\|_{L^2(\rho(t))}^2
    \,\mathrm dt
\right)^{1/2}.
\end{align*}

\begin{proof}[Proof of \Cref{prop:path-error}]
We first use the Stratonovich integral identity~\eqref{eq:StratExp} for the SDE sample paths to simplify the difference of the path-dependent QoI expectations:
\begin{align*}
    \left|
    \mathbb E[\Phi(X) ]
    -
    \mathbb E[\Phi(\hat{x}_\theta) ]
\right| &= \left| \mathbb{E} \left[\int_0^T \phi(t,X(t)) \circ \mathrm{d}X(t) - \int_0^T \left \langle \phi(t,\hat{x}_\theta(t)), \tfrac{\mathrm d}{\mathrm{d}t} \hat{x}_\theta(t) \right \rangle \mathrm{d}t \right] \right|\\
&= \left| \mathbb{E} \left[\int_0^T \bigg( \langle \phi(t,X(t)), v(t,X(t)) \rangle - \left \langle \phi(t,\hat{x}_\theta(t)), v_\theta(t,\hat{x}_\theta(t))\right \rangle \bigg) \mathrm{d}t \right] \right|\\
&\leq \int_0^T \left| \mathbb{E} \left[ \langle \phi(t,X(t)), v(t,X(t)) \rangle - \left \langle \phi(t,\hat{x}_\theta(t)), v_\theta(t,\hat{x}_\theta(t))\right \rangle \right] \right| \mathrm{d}t  
\end{align*}
Adding and subtracting
$
    \int_{\mathbb R^d}
    \langle \phi(t, x),v_\theta(t,x)\rangle
    \rho(t,x)\,\mathrm dx ,
$
we obtain
\begin{align}
  \left|
    \mathbb E[\Phi(X) ]
    -
    \mathbb E[\Phi(\hat{x}_\theta) ]
\right|
&\leq
\int_0^T \bigg(
\left|
    \int_{\mathbb R^d}
    \langle \phi(t,x),v(t,x)-v_\theta(t,x)\rangle
    \rho(t,x)\,\mathrm dx
\right|
\notag \\
&\qquad+
\left|
    \int_{\mathbb R^d}
    \langle \phi(t,x),v_\theta(t,x)\rangle
    \big(\rho(t,x)-\hat{\rho}_\theta(t,x)\big)\,\mathrm dx
\right| \bigg) \mathrm{d} t.
\label{eq:proof-rate-decomposition}
\end{align}
We bound the two terms separately. For the first term in~\eqref{eq:proof-rate-decomposition}, Cauchy--Schwarz gives
\begin{align*}
&
\left|
    \int_{\mathbb R^d}
    \langle \phi(t,x),v(t,x)-v_\theta(t,x)\rangle
    \rho(t,x)\,\mathrm dx
\right|
\notag \\
&\leq
\left(
    \int_{\mathbb R^d}
    \|v(t,x)-v_\theta(t,x)\|_2^2
    \rho(t,x)\,\mathrm dx
\right)^{1/2}
\left(
    \int_{\mathbb R^d}
    \|\phi(t,x)\|_2^2
    \rho(t,x)\,\mathrm dx
\right)^{1/2}
\notag \\
&=
\|v(t,\cdot)-v_\theta(t,\cdot)\|_{L^2(\rho(t))}
\|\phi(t,\cdot)\|_{L^2(\rho(t))}.
\end{align*}

For the second term in~\eqref{eq:proof-rate-decomposition}, define, for fixed
$t \in [0,T]$,
$
    f_t(x)
    :=
    \langle \phi(t, x),v_\theta(t,x)\rangle .
$
We show that $f_t$ is Lipschitz with $\operatorname{Lip}(f_t)\leq \tilde{C}$. For any $x,y\in\mathbb R^d$:
\begin{align*}
|f_t(x)-f_t(y)|
&=
\left|
\langle \phi(t,x),v_\theta(t,x)\rangle
-
\langle \phi(t, y),v_\theta(t,y)\rangle
\right|
\\
&\leq
\left|
\langle \phi(t, x)-\phi(t,y),v_\theta(t,x)\rangle
\right|
+
\left|
\langle \phi(t,y),v_\theta(t,x)-v_\theta(t,y)\rangle
\right|
\\
&\leq
\|\phi(t,x)-\phi(t,y)\|_2
\|v_\theta(t,x)\|_2
+
\|\phi(t,y)\|_2
\|v_\theta(t,x)-v_\theta(t,y)\|_2
\\
&\leq
L_{\phi}\|x-y\|_2 B_{v_\theta}
+
B_\phi L_{v_\theta}\|x-y\|_2
\\
&=
\tilde{C}\|x-y\|_2 .
\end{align*}
Then, by Kantorovich--Rubinstein duality and $W_1\leq W_2$,
\begin{align*}
\left|
    \int_{\mathbb R^d}
    \langle \phi(t, x),v_\theta(t,x)\rangle
    \big(\rho(t,x)-\hat{\rho}(t,x)\big)\,\mathrm dx
\right|
&=
\left|
    \int_{\mathbb R^d} f_t(x)\rho(t,x)\,\mathrm dx
    -
    \int_{\mathbb R^d} f_t(x)\hat{\rho}(t,x)\,\mathrm dx
\right|
\notag \\
&\leq
\operatorname{Lip}(f_t) W_1(\rho(t),\hat{\rho}(t))
\leq
\tilde{C} W_2(\rho(t),\hat{\rho}(t)).
\end{align*}
This finishes the proof of~\eqref{eq:qoi-rate-local-plus-marginal} from~\eqref{eq:proof-rate-decomposition}.

\end{proof}

\section{Details about experiments}\label{appendix:Experiments}
We provide the details for the four benchmark systems used in
\Cref{sec:NumExp}.  In all examples, FTM is trained from trajectory data
using the one-step loss \eqref{eq:FTM:OneStepEstimator}, except otherwise noted (in particular in \Cref{fig:duffing_combined}).  At test time, we
generate ensembles by integrating the learned ODE.

\paragraph{Trajectory-dependent probability-current QoIs}
We evaluate whether the learned model
recovers trajectory-dependent probability currents.  Given a smooth test
function
\[
    \phi : [0,T]\times \mathbb{R}^d \to \mathbb{R}^d,
\]
we define the Stratonovich path functional
\begin{equation}
    Q_\phi
    =
    \mathbb{E}\left[
    \int_0^T
    \phi(t,X(t)) \circ \mathrm dX(t)
    \right].
    \label{eq:appendix:qphi}
\end{equation}
For discrete trajectories
$\{X^{(i)}(t_k)\}_{i=1,k=0}^{N,K}$, we estimate \eqref{eq:appendix:qphi}
using the midpoint Stratonovich rule
\begin{equation}
    \widehat Q_\phi
    =
    \frac{1}{N}
    \sum_{i=1}^N
    \sum_{k=0}^{K-1}
    \phi\!\left(
        \frac{t_k+t_{k+1}}{2},
        \frac{X^{(i)}(t_k)+X^{(i)}(t_{k+1})}{2}
    \right)^\top
    \left(
        X^{(i)}(t_{k+1})-X^{(i)}(t_k)
    \right).
    \label{eq:appendix:qphi-discrete}
\end{equation}
The same estimator is applied to trajectories generated by each learned
model, and the reported trajectory-QoI error is the absolute difference from
the reference value computed on ground-truth trajectories.

\subsection{Duffing oscillator}\label{appendix:DuffingDetails}

\paragraph{Dynamics}
We consider the stochastically forced Duffing oscillator
\begin{align}
    \mathrm dX_1(t) &= X_2(t)\,\mathrm dt, \\
    \mathrm dX_2(t) &=
    \left(
        -2\xi\omega X_2(t)
        + \omega^2 X_1(t)
        - \omega^2\gamma X_1(t)^3
    \right)\mathrm dt
    + \sigma\,\mathrm dW(t),
\end{align}
with parameters
\[
    \xi = 0.2,
    \qquad
    \gamma = 0.2,
    \qquad
    \omega = 1,
    \qquad
    \sigma = 0.5.
\]
Equivalently, the deterministic force corresponds to a double-well potential
of the form
\[
    V(x_1) = -\frac12 x_1^2 + \frac{1}{20}x_1^4,
\]
up to the parameter scaling above.  The random forcing acts only on the
acceleration equation.  Initial conditions are sampled as
\[
    X(0) \sim
    \mathcal N
    \left(
        \begin{bmatrix}0 \\ -10\end{bmatrix},
        I_2
    \right).
\]

\paragraph{Probability-current QoI}
The Duffing oscillator has two metastable wells separated by the barrier near
$x_1=0$.  To measure the net probability current across this barrier, we use
the test function
\begin{equation}
    \phi(x)
    =
    \begin{bmatrix}
    \displaystyle
    \frac{1}{\sqrt{2\pi}\epsilon}
    \exp\!\left(-\frac{x_1^2}{2\epsilon^2}\right)
    \\
    0
    \end{bmatrix},
    \qquad
    \epsilon = 1.
    \label{eq:appendix:duffing-phi}
\end{equation}
This localizes the Stratonovich line integral near the separating surface
$x_1=0$ and therefore measures the signed probability flux moving between
the two wells.  We estimate $Q_\phi$ using
\eqref{eq:appendix:qphi-discrete}.

\paragraph{Training data generation} 
We generate training trajectories by integrating the Duffing Oscillator SDE with the Euler--Maruyama scheme on a uniform time grid with step size $\Delta t = 0.01$ over the time interval $[0, 12]$. Initial conditions are drawn from the distribution specified above, and we generate $ 5000$ training trajectories. For testing, we draw a new set of initial conditions from the same initial distribution and generate $5000$ testing trajectories. 

\paragraph{Reported metrics}
For the Duffing oscillator, we report both a distributional error between
time marginals and the trajectory-dependent current error associated with
\eqref{eq:appendix:duffing-phi}.  The distributional error is computed using
a sliced Wasserstein-$2$ distance between generated and reference ensembles,
averaged over the evaluated time points.

\subsection{Stochastic Rayleigh--Bénard convection}\label{appendix:RayleighDetails}

\paragraph{Dynamics}
We use the nine-dimensional chaotic Rayleigh--Bénard convection model of
\cite{reiterer-lainscek-schuerrer-etal:1998}.  The state is
\[
    X(t) = (X_1(t),\ldots,X_9(t)) \in \mathbb{R}^9,
\]
and the dynamics are given by
\begin{equation}
    \mathrm dX(t)
    =
    f_\mu(X(t))\,\mathrm dt
    +
    \sigma\,\mathrm dW(t),
    \label{eq:appendix:rb-sde}
\end{equation}
where $f_\mu$ is the deterministic nine-mode convection vector field and
$\mu$ is the Rayleigh-type control parameter.  In the experiments we use the
chaotic regime with
\[
    \sigma = \frac{1}{20},
\]
and treat $\mu$ as a control parameter that is varied across training and testing.

\paragraph{Rotational probability-current QoI}
The Rayleigh--Bénard dynamics exhibit persistent rotational motion in
low-dimensional projections.  To test whether a learned model reproduces this
rotational current, we consider the projection
\[
    P : \mathbb{R}^9 \to \mathbb{R}^2,
    \qquad
    Px = (x_1,x_2),
\]
and the counterclockwise rotation matrix
\[
    R =
    \begin{bmatrix}
        0 & -1 \\
        1 & 0
    \end{bmatrix}.
\]
We define
\begin{equation}
    \phi(x)
    =
    P^\top R Px
    =
    (-x_2,x_1,0,\ldots,0)^\top.
    \label{eq:appendix:rb-phi}
\end{equation}
The corresponding quantity $Q_\phi$ measures the component of the
probability current aligned with rotation in the $(x_1,x_2)$ plane.  As in
the Duffing example, we estimate it by the discrete Stratonovich estimator
\eqref{eq:appendix:qphi-discrete}.

\paragraph{Training data generation} 
We generate training trajectories by integrating the Rayleigh--Bénard SDE
\eqref{eq:appendix:rb-sde} with the Euler--Maruyama scheme on a uniform time
grid with step size $\Delta t = 0.01$ over the time interval $[0, 20]$.
Initial conditions are drawn from an isotropic Gaussian with mean $0$ and
standard deviation $0.02$ in each coordinate. For each training parameter
value
\[
    \mu \in \{13.5,\, 13.6,\, 13.7,\, 13.8,\, 13.9,\, 14.0,\, 14.1,\, 14.2\},
\]
we generate $20{,}000$ training trajectories, and we evaluate the learned
model at the held-out test value $\mu = 13.65$. For testing, we draw a new
set of initial conditions from the same Gaussian distribution and generate
$20{,}000$ testing trajectories at $\mu = 13.65$.

\paragraph{Reported metrics}
We report a distributional error between generated and reference time
marginals, together with the rotational-current error induced by
\eqref{eq:appendix:rb-phi}.  The latter distinguishes methods that reproduce
the correct projected distribution from those that also reproduce the
trajectory-induced circulation.

\subsection{Stochastically forced Burgers turbulence}
\label{appendix:BurgersExample}
\paragraph{Dynamics}
We consider the one-dimensional stochastic Burgers equation on the periodic
domain $[0,1)$,
\begin{equation}
    \mathrm d u(t,x)
    =
    \left(
        \nu\,\partial_{xx}u(t,x)
        -
        u(t,x)\,\partial_x u(t,x)
    \right)\mathrm dt
    +
    \sigma\,\mathrm dW(t,x),
    \qquad x\in[0,1),
    \label{eq:appendix:burgers}
\end{equation}
with viscosity
\[
    \nu = 0.007.
\]
The stochastic forcing is white in time and colored in space.  Specifically,
we use a truncated Fourier expansion on the first ten modes,
\[
    \mathrm dW(t,x)
    =
    \sum_{\iota=1}^{10}
    \frac{1}{\iota}
    \left(
        a_\iota(x)\,\mathrm dB_{\iota}^{(1)}(t)
        +
        b_\iota(x)\,\mathrm dB_{\iota}^{(2)}(t)
    \right),
\]
where $a_\iota$ and $b_\iota$ denote the corresponding sine and cosine modes,
and the Brownian motions are independent.  Unless otherwise stated, the noise
amplitude is
\[
    \sigma = \frac{2}{50}.
\]

\paragraph{Initial conditions}
Initial conditions are Gaussian bumps with additive colored perturbations,
\begin{equation}
    u(0,x)
    =
    \exp\!\left(-c(x-1/2)^2\right)
    +
    \eta(x),
    \label{eq:appendix:burgers-ic}
\end{equation}
where $\eta$ is drawn from the same spatially colored Fourier family as the
forcing, with perturbation strength $\sigma_0=0.015$.  The constant $c$
controls the width of the initial bump.

\paragraph{Training data generation} 
We solve \eqref{eq:appendix:burgers} using a method-of-lines discretization with second-order centered finite differences in space on a uniform grid of $N = 64$ points over the periodic domain $[0,1)$, combined with an Euler--Maruyama scheme in time with step size $\Delta t = 5\times 10^{-4}$
over the interval $[0, 4]$. Trajectories are subsampled to $T = 800$ uniformly spaced output times (output spacing $\Delta t_{\text{out}} = 5\times 10^{-3}$). Initial conditions are drawn according to \eqref{eq:appendix:burgers-ic} with bump width $c = 20$. We generate $4{,}096$ training trajectories and $4{,}096$ testing trajectories, with testing initial conditions drawn independently from the same distribution.

\paragraph{Physical diagnostics}
For Burgers, the main evaluation metrics are the energy and enstrophy of the
field,
\begin{align}
    E(t)
    &=
    \frac12
    \int_0^1
    |u(t,x)|^2\,\mathrm dx,
    \\
    Z(t)
    &=
    \frac12
    \int_0^1
    |\partial_x u(t,x)|^2\,\mathrm dx.
\end{align}
On a uniform grid, these quantities are computed by the corresponding
trapezoidal or periodic finite-difference approximations.  We compare the ensemble mean and ensemble standard deviation of $E(t)$ and $Z(t)$ over
generated and reference trajectories.

\paragraph{Noise-sensitivity experiment}
To test sensitivity to stochastic forcing, we repeat the Burgers experiment
for multiple values of the forcing amplitude $\sigma$.  This comparison is
used to verify that the benchmark contains genuine stochastic variability:
standard deterministic time-steppers trained with mean-squared error tend to
collapse toward conditional-mean dynamics as $\sigma$ increases, whereas FTM
is evaluated by its ability to maintain accurate ensemble statistics.

\subsection{Stochastically forced turbulence (Navier-Stokes)}
\label{appendix:Experiments:NavierStokes}

\paragraph{Dynamics and numerical solver}
We consider two-dimensional incompressible Navier--Stokes flow on the
periodic domain
\[
    \Omega = [0,2\pi]^2.
\]
The simulations are based on the spectral forced-turbulence example from
JAX-CFD~\cite{Kochkov2021-ML-CFD,Dresdner2022-Spectral-ML}, modified to include additive stochastic forcing in place of deterministic Kolmogorov forcing. We use the
vorticity formulation
\begin{equation}
    \mathrm d\omega
    +
    (u\cdot \nabla \omega)\,\mathrm dt
    =
    \left(\nu\,\Delta \omega - \alpha\,\omega\right)\mathrm dt
    +
    f_{\mathrm{sto}}(t,x)\,\mathrm dt,
    \qquad
    u = \nabla^\perp \Delta^{-1}\omega,
    \label{eq:appendix:ns-vorticity}
\end{equation}
with viscosity $\nu = 10^{-3}$ and linear drag coefficient $\alpha = 0.1$. 

\paragraph{Stochastic coarse-mode forcing}
The driving stochastic forcing acts on a finite set of low-frequency Fourier modes. In vorticity form, the stochastic forcing is
\begin{equation}
    f_{\mathrm{sto}}(t,x)\,\mathrm dt
    =
    \sigma
    \sum_{\kappa \in \mathcal K_{\mathrm{sto}}}w_\kappa
    \left(
        a_\kappa(x)\,\mathrm dB_\kappa^{(1)}(t)
        +
        b_\kappa(x)\,\mathrm dB_\kappa^{(2)}(t)
    \right),
    \label{eq:appendix:ns-stochastic-forcing}
\end{equation}
where $\mathcal K_{\mathrm{sto}}$ contains only coarse Fourier modes,
$a_\kappa$ and $b_\kappa$ are the corresponding real sine and cosine Fourier
basis functions, the Brownian motions are independent, and the spectral weights are $w_\kappa \propto |\kappa|^{-1/2}$. Thus the unresolved random forcing acts at large spatial scales, while the nonlinear Navier--Stokes dynamics transfer this variability across scales. In all our experiments, we set the stochastic forcing strength to be $\sigma = 0.3$.

\paragraph{Initial conditions}
Trajectories are initialized from filtered random velocity fields as in the
JAX-CFD forced-turbulence setup, rescaled so that the maximum velocity
magnitude is approximately $7$.  The combination of random initial
conditions and stochastic coarse-mode forcing yields an ensemble of turbulent rollouts.

\paragraph{Training data generation}
Equation~\eqref{eq:appendix:ns-vorticity} is integrated on a $256 \times 256$ uniform grid over $[0, 2\pi]^2$ using a pseudo-spectral method with $2/3$-rule dealiasing. Time integration is performed via an implicit--explicit scheme: the linear (viscous and drag) terms are treated implicitly with Crank--Nicolson, while advection and stochastic forcing are advanced explicitly with a five-stage low-storage Carpenter--Kennedy Runge--Kutta-4 scheme. We use an internal timestep $\Delta t_{\mathrm{sim}} = 10^{-3}$ and integrate up to $T_{\mathrm{end}} = 25$. Each trajectory is then subsampled in time to $T = 1000$ evenly-spaced snapshots and spectrally subsampled in space to a $64 \times 64$ output grid. We generate $2048$ training trajectories and a separate set of $1024$ test trajectories.

\paragraph{Physical diagnostics} Same as for the Burgers example; see \Cref{appendix:BurgersExample}. 

\subsection{Setup and details of numerical experiments}

\paragraph{Architectures for FTM}

\paragraph{Duffing oscillator.}
We use a 3-layer MLP with hidden dimension $128$ and ReLU activations. Time is provided by concatenation to the input. Training uses Adam with learning rate $5\times10^{-4}$ under a warmup--cosine-decay schedule, batch size $8192$, for $10^5$ steps.

\paragraph{Rayleigh--B\'enard.}
We use a 7-layer MLP with hidden dimension $128$ and ReLU activations. Time is provided by concatenation to the input. Training uses Adam with learning rate $5\times10^{-4}$ under a warmup--cosine-decay schedule, batch size $8192$, for $10^5$ steps.

\paragraph{Stochastic Burgers.}
We use a 1D U-Net with channel widths $[32, 64, 128]$, two residual blocks per scale, one middle residual block, and GeLU activations. Time is embedded via a 2-layer MLP of width $128$ and injected into each residual block via a learned linear projection: the projected vector is added to the feature map both before the first convolution and before the second convolution, broadcast across the spatial dimension. Training uses Adam with learning rate $1\times10^{-4}$ under a warmup--cosine-decay schedule, batch size $2048$, for $2\times10^5$ steps.

\paragraph{Stochastic Navier--Stokes.}
We use a 2D U-Net with channel widths $[128, 256, 512, 1024]$, two residual blocks per scale, one middle residual block, and GeLU activations, conditioned on three consecutive frames. Time is embedded via a 2-layer MLP of width $512$ and injected into each residual block via a learned linear projection: the projected vector is added to the feature map both before the first convolution and before the second convolution, broadcast across the spatial dimensions. Training uses Adam with learning rate $1\times10^{-4}$ under a warmup--cosine-decay schedule, batch size $128$, for $10^6$ steps.

\paragraph{Compute}
We train all models, including the other baselines, on compute nodes with H200 GPUs and 60GB of memory and 2TB of node main memory; and also use the same H200 GPUs for inference tasks. All runtime results are measured on H200 GPUs. Batch sizes are chosen so that they fit into the memory of the GPUs. 

\subsection{Additional description and details about experiments}

\paragraph{\Cref{fig:duffing_combined}} The cost reported in panel (c) is the runtime of one loss evaluation, which has been averaged over 100 runs. The shaded area shows the standard deviation over these 100 runs. The time-step size $h = 10^{-2}$ and fixed over all $\tau$, which is the setup used in the Duffing oscillator experiment.

\paragraph{\Cref{tab:duffing_lorenz}} The $W_2$ error for the Duffing and Rayleigh-Benard convection is computed for each time step separately and then these errors are averaged, which is the reported number. The number in brackets is the standard deviation of the error over all time steps, which indicates how robustly the time marginals are matched over time. For the trajectory-based QoI error, the QoI is computed (\Cref{appendix:DuffingDetails,appendix:RayleighDetails}) over all test trajectories. The reported number is the average relative error. The number in brackets is the standard deviation of the Monte Carlo estimator.  

\paragraph{\Cref{tab:maintextburgers_jaxcfd} and \Cref{tab:burgers_jaxcfd}} For computing the energy and enstrophy estimates, see \Cref{appendix:BurgersExample}. We report the relative error averaged over the whole time range. In brackets, we report the standard deviation of the errors over time. 

\paragraph{\Cref{fig:duffing_histogram}(right)} The right panel shows the error of energy and enstrophy of FTM for the Burgers example for varying noise level $\sigma$. 

\paragraph{\Cref{fig:jax_cfd_combined}} The spectrum $E(k)$ shown in the plots on the right is computed by taking the Fourier decomposition of the generated field at times $t = 15$ and $t = 25$ and then plotting the coefficients.

\subsection{Additional experiments} \Cref{tab:burgers_jaxcfd} shows more extensive results for the stochastic Burgers example and the stochastically forced turbulence example. \Cref{tab:meanflowext} shows extensive results for using a distilled auto-regressive conditional flow matching (AR-CFM) model, using mean-flows \cite{geng2025mean} for distillation. 

\Cref{fig:burgers_combined_appendix} shows the mean and standard deviation of energy and enstrophy of the ensemble trajectories compared to the ground truth for stochastic Burgers and stochastic turbulence examples.

\begin{figure}
\hspace*{-0.25cm}\begin{tabular}{cc}
\includegraphics[width=0.48\linewidth]{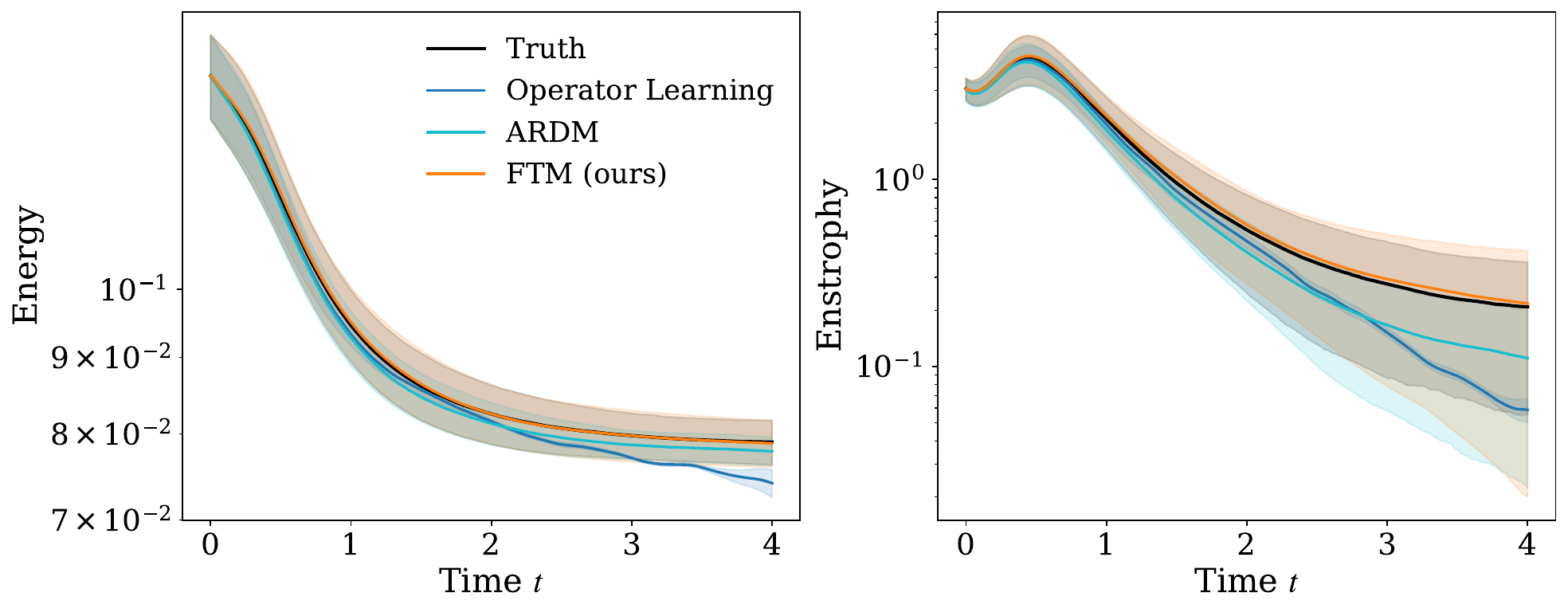} & \includegraphics[width=0.48\linewidth]{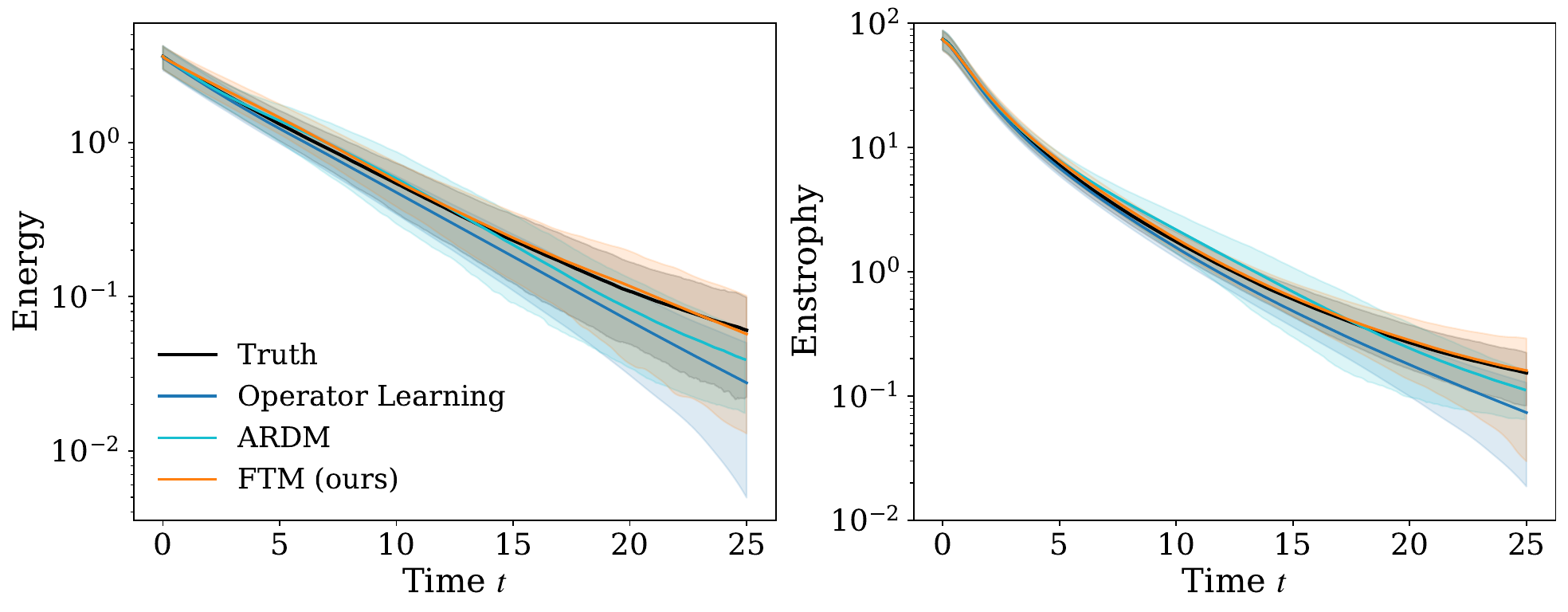}\\
\scriptsize (a) stochastic Burgers turbulence (Burgulence) & \scriptsize (b) stochastically forced turbulence (Navier-Stokes)
\end{tabular}
    \caption{FTM predicts ensembles with accurate energy and enstrophy statistics of stochastic Burgers (left panel) and stochastically forced turbulence (right panel) solution trajectories.} 
    \label{fig:burgers_combined_appendix}
\end{figure}

\begin{table}
\centering\small
\caption{Comparison of methods on the stochastic Burgers equation and stochastically forced turbulence (two-dim.\ Navier-Stokes).}
\label{tab:burgers_jaxcfd}
\begin{tabular}{lccc}
\toprule
\multicolumn{4}{c}{\textbf{Stochastic Burgers}} \\
\midrule
\textbf{Method} & error energy & error enstrophy & NFEs/time step \\
\midrule
Operator learning \cite{stachenfeld2021learned} 
& $2.21\mathrm{e}{-2}\;(\pm\,1.54\mathrm{e}{-2})$ 
& $2.55\mathrm{e}{-1}\;(\pm\,2.34\mathrm{e}{-1})$ 
& 1 \\

ARDM 50 steps\cite{KOHL2026108641} 
& $1.46\mathrm{e}{-2}\;(\pm\,3.39\mathrm{e}{-3})$ 
& $2.49\mathrm{e}{-1}\;(\pm\,1.43\mathrm{e}{-1})$ 
& 50 \\

ARDM 75 steps \cite{KOHL2026108641} 
& $1.36\mathrm{e}{-2}\;(\pm\,3.49\mathrm{e}{-3})$ 
& $2.30\mathrm{e}{-1}\;(\pm\,1.26\mathrm{e}{-3})$ 
& 75 \\

ARDM 100 steps \cite{KOHL2026108641} 
& $1.24\mathrm{e}{-2}\;(\pm\,3.12\mathrm{e}{-3})$  
& $2.11\mathrm{e}{-1}\;(\pm\,1.14\mathrm{e}{-1})$  
& 100 \\
CFM 20 steps \cite{albergo-vanden-eijnden:2022,lipman2023flow} & $2.71\mathrm{e}{-3}\;(\pm\,1.96\mathrm{e}{-3})$ 
& $1.53\mathrm{e}{-1}\;(\pm\,1.28\mathrm{e}{-1})$ 
& 20 \\

\textbf{FTM (ours)} 
& $\mathbf{1.95\mathrm{e}{-3}}\;(\pm\,1.93\mathrm{e}{-3})$ 
& $\mathbf{5.32\mathrm{e}{-2}}\;(\pm\,2.23\mathrm{e}{-2})$ 
& 1 \\

\midrule
\multicolumn{4}{c}{\textbf{stochastically forced turbulence (2D Navier-Stokes)}} \\
\midrule
\textbf{Method} & error energy & error enstrophy & NFEs/time step \\
\midrule
Operator learning \cite{stachenfeld2021learned} 
& 0.213 
& 0.192 
& 1 \\

ARDM 50 steps \cite{KOHL2026108641} 
& $0.314 \pm 0.1478$
& $0.282 \pm 0.0960$
& $50$ \\

ARDM 75 steps \cite{KOHL2026108641} 
& $0.275 \pm 0.1403$
& $0.171 \pm 0.0939$
& $75$ \\
ARDM 100 steps \cite{KOHL2026108641} 
& $0.116 \pm 0.1092$
& $0.134 \pm 0.0846$
& $100$ \\
CFM 20 steps
& $0.261 \pm 0.272$
& $0.271 \pm 0.326$
& 20 \\
MeanFlow 1 step
& $0.868 \pm 0.763$
& $1.230 \pm 1.289$
& $1$ \\
MeanFlow 2 steps
& $0.787 \pm 0.677$
& $0.665 \pm 0.654$
& $2$ \\
MeanFlow 4 steps
& $0.700 \pm 0.624$
& $0.588 \pm 0.622$
& $4$ \\
MeanFlow 8 steps
& $0.590 \pm 0.517$
& $0.485 \pm 0.524$
& $8$ \\
\textbf{FTM (ours)}
& $\mathbf{0.034 \pm 0.025}$
& $\mathbf{0.033 \pm 0.019}$
& 1 \\
\bottomrule
\end{tabular}
\end{table}

\section{Baselines}\label{appendix:Baselines}
\label{appendix:DirectEstimationOfV}

All baselines are trained on the respective same data sets and comparable architectures as FTM to ensure a fair comparison.

\paragraph{Neural time stepper}
We follow the standard learned-simulator baseline used in, e.g., \cite{otness2021an} and fit a deterministic one-step map $f_\theta : (t_k,x_k) \mapsto x_{k+1}$ 
by minimizing the regression loss
\[
    \mathcal L_{\mathrm{step}}(\theta)
    =
    \frac{1}{N K}
    \sum_{i=1}^{N}
    \sum_{k=0}^{K-1}
    \left\|
        f_\theta(t_k,X^{(i)}(t_k))
        -
        X^{(i)}(t_{k+1})
    \right\|_2^2 .
\]
At inference time, trajectories are generated autoregressively by setting $\hat x(t_{k+1}) = f_\theta(t_k,\hat x(t_k))$.  For stochastic data, this baseline estimates the conditional mean of the next state under squared-error training analogous to operator learning, and therefore serves as a deterministic time-stepper baseline.  

For each of the experiments, we use the same architecture and hyperparameter sets mentioned in \Cref{appendix:Experiments}.

\paragraph{SDE learning}
  We implement the neural Euler--Maruyama likelihood baseline of \cite{dridi2021learning} for the Duffing and Rayleigh--B\'enard      problems.                                                                                                                 
  Two MLPs are trained jointly: a drift network $b_\theta(x, t)$ and a diffusion network $\sigma_\theta(t)$, the latter     
  taking only time as input with a softplus output to enforce positivity.                                                   
  Given consecutive trajectory pairs $(X_k, X_{k+1})$ with step size $\Delta t$, the model minimizes the negative           
  log-likelihood under the Euler--Maruyama Gaussian transition $X_{k+1} \mid X_k \sim \mathcal{N}(X_k + b_\theta(X_k,
  t_k),\Delta t,; \sigma_\theta(t_k)^2,\Delta t\cdot I)$:                                                                   
  $$\mathcal{L} = \mathbb{E}\left[\log\left(\sigma_\theta^2,\Delta t\right) + \frac{|X_{k+1} - X_k - b_\theta,\Delta
  t|^2}{\sigma_\theta^2,\Delta t}\right].$$

After training, rollouts are generated by simulating the learned SDE with Euler--Maruyama.  This baseline therefore learns a full stochastic model, rather than only the probability current velocity.

The hyperparameters for the MLPs for each of the experiments are the same as mentioned in Appendix \Cref{appendix:Experiments}.

\paragraph{SDE matching} We use the codebase published by the authors of the original paper \cite{bartosh2025sdematching} on GitHub. The model is trained in the state space, which is tractable for the Duffing oscillator and Rayleigh--Bénard convection. For the Duffing oscillator $128$ hidden dimensions were used with $3$ layers, and for the Rayleigh--Bénard problem $128$ hidden dimensions with $7$ layers was used. The learning rate for both was $5 \times 10^{-4}$, and the latent dimension for each was set to be equal to the state dimension ($2$ for Duffing oscillator, $9$ for Rayleigh--Bénard convection). The Duffing oscillator was trained for $150,000$ training steps and the Rayleigh--Bénard convection was trained for $300,000$ training steps.

\paragraph{Direct (``plug-in'') estimation of $v$}
We estimate the drift and diffusion coefficient as in SDE learning (see paragraph above) and additionally estimate the score conditioned on time with \cite{ho2020denoising} from the samples of the time marginals. The setup for estimating the drift and diffusion coefficient is the same as in the SDE learning (see paragraph above). For learning the score, we use denoising score matching. Concretely, we train a time-conditioned denoising network with the $v$-parametrization on $T = 1000$ diffusion steps with a linear $\beta$-schedule. At inference we query at a small diffusion step $t = 3$ and recover the score. The score network uses the same MLP backbone as the drift and diffusion networks.

\paragraph{DICE} We directly used the code that comes with the paper \cite{blickhan-berman-stuart-etal:2025}. For the duffing oscillator, we use an MLP backbone with $128$ hidden dimensions and $5$ layers with swish activation functions. We use learning rate $5 \times 10^{-4}$ with a cosine scheduler, and batchsize $256 \times 256$. For the Rayleigh--B\'enard system, we use the same hyperparameters and architecture in the DICE paper.

\paragraph{Autoregressive conditional diffusion model (ARDM)} For the Burgers equation example, we implement an ARDM baseline in the same spirit as \cite{KOHL2026108641}. The model learns the conditional transition law $X(t_{k + 1}) | X(t_k)$ rather than a deterministic one-step map.  Concretely, we train a conditional score network with the denoising score-matching setup.  At inference time, a next state is generated by initializing the sampler from the reference distribution and numerically integrating the learned reverse-time sampling dynamics conditioned on the current state $x_k$.  Repeating this procedure autoregressively produces a trajectory. For the two-dimensional Navier--Stokes experiments in \Cref{appendix:Experiments:NavierStokes}, we directly use the codebase released with \cite{KOHL2026108641}.  

For the Burgers implementation, we used a 1D UNet with feature depths $[64, 128, 256]$, four residual blocks per scale, one middle residual block, and GeLU activations. The noise schedule is the linear $\beta$-schedule with $T = 1000$ levels. The loss is the same as the denoising score matching via the v-parametrization. The optimizer used is Adam with learning rate $1 \times 10^{-4}$, with a warmup and then cosine decay. Training is done for $650,000$ gradient steps. At inference, DDIM sampling is used with $\eta = 1.0$, thus making it stochastic. 

For Navier Stokes, we use a 2D U-Net with ConvNext blocks and feature depths $[42, 64, 128, 256]$, two ConvNext residual blocks per scale, and GELU activations. Each ConvNext block applies a depthwise $7{\times}7$ convolution followed by two  $3{\times}3$ convolutions with GroupNorm and GELU. Linear attention is applied at each resolution level and full self-attention at the bottleneck. The network is conditioned on three consecutive vorticity frames (clean conditioning: conditioning frames are passed without noise during both training and inference); the noisy target frame is concatenated to the conditioning frames to form the 4-channel UNet input. Time is embedded via sinusoidal embeddings projected by a two-layer MLP (widths $[64, 256]$ with GELU), and injected into each ConvNext block via a learned linear projection added to the feature map after the depthwise convolution, broadcast across spatial dimensions. The forward diffusion uses a linear noise schedule with $T{=}100$ steps ($\beta_0{=}5{\times}10^{-4}$, $\beta_T{=}0.1$, scaled from the standard 1000-step parametrisation). Training uses Adam with learning rate $1{\times}10^{-4}$, no weight decay, batch size $64$, for $2,000,000$ training steps. Inference is done in the same way as in Burgers.

\begin{table}
\centering\small
\caption{Comparison of distillation with mean flow to FTM on the stochastic Burgers equation.}
\label{tab:meanflowext}
\begin{tabular}{lccc}
\toprule
\textbf{Method} & error energy & error enstrophy & NFEs/time step \\
\midrule
CFM 20 steps \cite{albergo-vanden-eijnden:2022,lipman2023flow} & $2.71\mathrm{e}{-3}\;(\pm\,1.96\mathrm{e}{-3})$ 
& $1.53\mathrm{e}{-1}\;(\pm\,1.28\mathrm{e}{-1})$ 
& 20 \\
MeanFlow 1 step\cite{geng2025mean} & $8.22\mathrm{e}{-1}\;(\pm\,2.88\mathrm{e}{-1})$ 
& $359\;(\pm\,278)$ 
& 1 \\
MeanFlow 2 steps \cite{geng2025mean} & $1.69\mathrm{e}{-1}\;(\pm\,5.77\mathrm{e}{-2})$ 
& $98.1\;(\pm\,80.5)$ 
& 2 \\
MeanFlow 5 steps \cite{geng2025mean} & $3.41\mathrm{e}{-2}\;(\pm\,2.22\mathrm{e}{-2})$ 
& $3.15\mathrm{e}{-1}\;(\pm\,1.81\mathrm{e}{-1})$ 
& 5 \\
MeanFlow 10 steps\cite{geng2025mean} & $2.06\mathrm{e}{-2}\;(\pm\,1.58\mathrm{e}{-2})$ 
& $3.15\mathrm{e}{-1}\;(\pm\,1.81\mathrm{e}{-1})$ 
& 10 \\
MeanFlow 20 steps \cite{geng2025mean} & $9.71\mathrm{e}{-3}\;(\pm\,2.91\mathrm{e}{-3})$ 
& $2.46\mathrm{e}{-1}\;(\pm\,2.57\mathrm{e}{-1})$ 
& 20 \\
\textbf{FTM (ours)} 
& $\mathbf{1.95\mathrm{e}{-3}}\;(\pm\,1.93\mathrm{e}{-3})$ 
& $\mathbf{5.32\mathrm{e}{-2}}\;(\pm\,2.23\mathrm{e}{-2})$ 
& 1 \\

\bottomrule
\end{tabular}
\end{table}

\paragraph{Autoregressive conditional flow matching and mean-flow model} Analogously to ARDM, we implement an autoregressive conditional generative baseline, but replace the reverse diffusion sampler by a conditional flow matching (CFM) ODE~\cite{albergo-vanden-eijnden:2022,lipman2023flow}. We also implemented a mean-flow-style distillation baseline \cite{geng2025mean,boffi2025how}, but we used a two-stage procedure. We first train the autoregressive CFM model and then distill its sampling-time flow into a model of the average velocity.

We train the CFM model to map the standard Gaussian to the target distribution, which is the conditional transition law. For each training pair $(x_t, x_{t+1})$ we draw $z \sim \mathcal{N}(0,I)$ and a flow time $\tau \sim \mathrm{Uniform}(0,1)$, construct the interpolant $x_\tau = \tau x_{t+1} + (1{-}\tau)z$, and regress the vector field onto the straight-line target $u_\tau = x_{t+1} - z$:
\begin{equation}
    \mathcal{L}(\theta)
    = \mathbb{E}_{\tau,\,z,\,(x_t,x_{t+1})}
      \bigl[\|v_\theta(x_\tau,\tau \mid x_{\mathrm{cond}}) - \operatorname{sg}(u_\tau)\|^2\bigr].
\end{equation}
For the Stochastic Burgers experiments, we use a 1D U-Net with channel widths $[96, 192, 384, 384]$, one residual block per encoder and decoder level, one bottleneck block, and SiLU activations; all convolutions apply circular padding to respect the periodic spatial domain, and skip connections are merged by concatenation. Time $\tau$ is embedded via a 128-dimensional sinusoidal encoding projected to dimension 512 through a  two-layer MLP (Linear $\to$ SiLU $\to$ Linear). Training uses Adam with gradient clipping (global norm $1.0$), a peak learning rate of $3 \times 10^{-4}$ under a cosine decay schedule with a 1000-step linear warm-up, batch size $256$, and runs for $500{,}000$ gradient steps. \\ \\
For Navier Stokes experiments, we use a 2D U-Net with channel widths $[128, 256, 512, 512]$ (doubling per level, capped at $4\times$ base), one residual block        
  (SiLU, $3{\times}3$ convolution) per encoder and decoder level, and skip connections merged by channel
  concatenation. Time $\tau$ is embedded via a 128-dimensional sinusoidal encoding
  projected to dimension 512 through a two-layer MLP (Linear $\to$ SiLU $\to$ Linear), and conditioning is done by simple channel concatenation of the prior 3 frames; model parameters are
  maintained with an exponential moving average of decay $0.9999$ and EMA weights are used at inference. Training uses Adam with
  gradient clipping (global norm $1.0$), a peak learning rate of $3 \times 10^{-4}$ under a cosine decay schedule with a 1000-step
  linear warm-up, batch size $256$, and runs for $1{,}000{,}000$ gradient steps.

We distill the trained CFM vector field $v_\theta$ into a mean-velocity
network $u_\psi(x_r, r, t \mid x_{\mathrm{cond}})$ following the
ODE-regression formulation of~\cite{geng2025mean}:
\begin{equation}
    u_\psi(x_r, r, t) \;\approx\; \frac{\hat{x}_t - x_r}{t - r},
\end{equation}
where $\hat{x}_t$ is the endpoint obtained by integrating the \emph{frozen}
oracle $v_\theta$ forward from $x_r$ over the interval $[r,t]$.
At inference a single network evaluation suffices:
$x_1 = z + u_\psi(z, 0, 1 \mid x_{\mathrm{cond}})$, but one can do multiple evaluations uniformly along the $[0,1]$ interval.
The MeanFlow loss is a regression task of the form:
\begin{equation}
    \mathcal{L}(\psi)
    = \mathbb{E}_{r \leq t,\;z,\;(x_t,x_{t+1})}
      \bigl[\|u_\psi(x_r, r, t \mid x_{\mathrm{cond}})
             - \operatorname{sg}(u_{\mathrm{tgt}})\|^2\bigr].
\end{equation}

The network $u_\psi$ shares the U-Net body of $v_\theta$ exactly.
The only architectural change is the time-conditioning module: both
interval endpoints $r$ and $t$ are embedded independently with the same
sinusoidal encoding of dimension 128, concatenated to a 256-dimensional
vector, and projected to 512 via a two-layer MLP. Training uses Adam with a peak learning rate of $10^{-4}$ under a cosine
decay schedule with a 10000-step warm-up, batch size 256, and runs for
1,000,000 gradient steps.

\begin{table}[h]
\centering
\caption{MeanFlow performance metrics across number of steps.}
\label{tab:meanflow}
\begin{tabular}{cccccc}
\hline
\textbf{Steps} & \textbf{Energy Err} & \textbf{Energy Err} & \textbf{Enstrophy Err} & \textbf{Enstrophy Err} & \textbf{Gen Time} \\
 & \textbf{(mean)} & \textbf{(std)} & \textbf{(mean)} & \textbf{(std)} & \textbf{(s)} \\
\hline
1 & 0.8687 & 0.7632 & 1.2308 & 1.2885 & 84.5 \\
2 & 0.7874 & 0.6770 & 0.6649 & 0.6539 & 143.8 \\
4 & 0.7003 & 0.6239 & 0.5884 & 0.6219 & 266.2 \\
8 & 0.5899 & 0.5171 & 0.4849 & 0.5235 & 514.4 \\
\hline
\end{tabular}
\end{table}

\end{document}